\documentclass{article}

\usepackage[final]{neurips_2023}

\usepackage{natbib}
\usepackage{shortcuts}
\usepackage{wrapfig}

\input{rwstyle.tex}

\newcommand{\LSEC}{\op{LSEC}}
\newcommand{\DE}{\op{DE}}

\newcommand{\online}{\textsc{O-DISCO}}
\newcommand{\offline}{\textsc{P-DISCO}}
\newcommand{\cb}{\textsc{DistCB}}

\usepackage[utf8]{inputenc} %
\usepackage[T1]{fontenc}    %
\usepackage{hyperref}       %
\usepackage{url}            %
\usepackage{booktabs}       %
\usepackage{amsfonts}       %
\usepackage{nicefrac}       %
\usepackage{microtype}      %
\usepackage{xcolor}         %

\usepackage{thmtools,thm-restate} %
\crefname{assumption}{Assumption}{Assumptions}

\usepackage{tabularx}
\makeatletter
\newcommand{\multiline}[1]{%
  \begin{tabularx}{\dimexpr\linewidth-\ALG@thistlm}[t]{@{}X@{}}
    #1
  \end{tabularx}
}
\makeatother

\title{The Benefits of Being Distributional: \\ Small-Loss Bounds for Reinforcement Learning}

\author{%
Kaiwen Wang \quad Kevin Zhou \quad Runzhe Wu \quad Nathan Kallus \quad Wen Sun \\
Cornell University \\
\texttt{\{kw437,klz23,rw646,kallus,ws455\}@cornell.edu}
}

\begin{document}

\maketitle

\begin{abstract}
While distributional reinforcement learning (DistRL) has been empirically effective, the question of when and why it is better than vanilla, non-distributional RL has remained unanswered.
This paper explains the benefits of DistRL through the lens of small-loss bounds, which are instance-dependent bounds that scale with optimal achievable cost.
Particularly, our bounds converge much faster than those from non-distributional approaches if the optimal cost is small.
As warmup, we propose a distributional contextual bandit (DistCB) algorithm, which we show enjoys small-loss regret bounds and empirically outperforms the state-of-the-art on three real-world tasks.
In online RL, we propose a DistRL algorithm that constructs confidence sets using maximum likelihood estimation. We prove that our algorithm enjoys novel small-loss PAC bounds in low-rank MDPs.
As part of our analysis, we introduce the $\ell_1$ distributional eluder dimension which may be of independent interest.
Then, in offline RL, we show that pessimistic DistRL enjoys small-loss PAC bounds that are novel to the offline setting and are more robust to bad single-policy coverage.
\end{abstract}

\section{Introduction}
The goal of reinforcement learning (RL) is to learn a policy that minimizes/maximizes the mean loss/return (\ie, cumulative costs/rewards) along its trajectory.
Classical approaches, such as $Q$-learning \citep{mnih2015human} and policy gradients \citep{kakade2001natural}, often learn $Q$-functions via least square regression, which represent the mean loss-to-go and act greedily with respect to these estimates. By Bellman's equation, $Q$-functions suffice for optimal decision-making and indeed these approaches have vanishing regret bounds, suggesting we only need to learn means well \citep{sutton2018reinforcement}.
Since the seminal work of \citet{bellemare2017distributional}, however, numerous developments
showed that learning the \emph{whole} loss distribution
can actually yield state-of-the-art performance in stratospheric balloon navigation \citep{bellemare2020autonomous}, robotic grasping \citep{bodnar2020quantile}, algorithm discovery \citep{fawzi2022discovering} and game playing benchmarks \citep{hessel2018rainbow,dabney2018implicit,barth-maron2018distributional}.
In both online \citep{yang2019fully} and offline RL \citep{ma2021conservative}, distributional RL (DistRL) algorithms often perform better and use fewer samples in challenging tasks when compared to standard approaches that directly estimate the mean.

Despite learning the whole loss distribution,
DistRL algorithms use only the mean of the learned distribution for decision making, not extracting any additional information such as higher moments.
In other words, DistRL is simply employing a different and seemingly roundabout way of learning the mean:
first, learn the loss-to-go distribution via distributional Bellman equations, and then, compute the mean of the learned distribution.
\citet{lyle2019comparative} provided some empirical explanations of the benefits of this two-step approach,
showing that learning the distribution, \eg, its moments or quantiles, is an auxiliary task that leads to better representation learning.
However, the theoretical question remains: does DistRL, \ie, learning the distribution and then computing the mean, yield provably stronger finite-sample guarantees and if so stronger how and when?

In this paper, we provide the first mathematical basis for the benefits of DistRL via the lens of small-loss bounds, which are instance-dependent bounds that depend on the minimum achievable cost in the problem \citep{agarwal2017open}.\footnote{``First-order'' generally refers to bounds that scale with the optimal value, either the maximum reward or the minimum cost. To highlight that we are minimizing cost, we call our bounds ``small-loss''.}
For example in linear MDPs, typical worst-case regret bounds scale on the order of $\op{poly}(d,H)\sqrt{K}$, where $d$ is the feature dimension, $H$ is the horizon, and $K$ is the number of episodes \citep{jin2020provably}.
In contrast, small-loss bounds will scale on the order of $\op{poly}(d,H)\sqrt{K \cdot V^\star} + \op{poly}(d,H)\log(K)$, where $V^\star = \min_{\pi}V^\pi$ is the optimal expected cumulative cost for the problem. %
We assume cumulative costs are normalized in $[0,1]$ without loss of generality.
As $V^\star$ becomes negligible (approaches $0$), the first term vanishes and the small-loss bound yields a faster convergence rate of $\Ocal(\op{poly}(d,H)\log(K))$, compared to the $\Ocal(\op{poly}(d,H)\sqrt{K})$ rate in standard uniform bounds.
Since we always have $V^\star\leq 1$, small-loss bounds simply match the standard uniform bounds in the worst case.

As warm-up, we show that maximum likelihood estimation (MLE), \ie, maximizing log-likelihood, can be used to obtain small-loss regret bounds for contextual bandits (CB), \ie, the one-step RL setting.
Then, we turn to the online RL setting, and propose an optimistic DistRL algorithm that optimizes over confidence sets constructed via MLE applied to the distributional Bellman equations.
We prove our algorithm attains the first small-loss PAC bounds in low-rank MDPs \citep{agarwal2020flambe}. Our proof uses a novel regret decomposition with triangular discrimination and also introduces the $\ell_1$ distributional eluder dimension, which generalizes the $\ell_2$ distributional eluder dimension of \citet{jin2021bellman} and may be of independent interest.
Furthermore, we design an offline distributional RL algorithm using the principle of pessimism, and show our algorithm obtains the first small-loss bounds in offline RL. Our offline small-loss bound holds under the weak single-policy coverage. Notably, our result has a novel robustness property that allows our algorithm to strongly compete with policies that either are well-covered or have small-loss, while prior approaches solely depended on the former.
Finally, we find that our distributional CB algorithm empirically outperforms existing approaches in three challenging CB tasks.

Our key contributions are as follows:
\begin{enumerate}[leftmargin=0.7cm]
    \item As warm-up, we propose a distributional CB algorithm and prove that it obtains a small-loss regret bound (\cref{sec:warm-up}).
    We empirically demonstrate it outperforms state-of-the-art CB algorithms in three challenging benchmark tasks (\cref{sec:experiments}).
    \item We propose a distributional online RL algorithm that enjoys small-loss bounds in settings with low $\ell_1$ distributional eluder dimension, which we show can always capture low-rank MDPs. The $\ell_1$ distributional eluder dimension may be of independent interest
    (\cref{sec:online-rl}).
    \item We propose a distributional offline RL algorithm and prove that it obtains the first small-loss bounds in the offline setting. Our small-loss guarantee exhibits a novel robustness to bad coverage, which implies strong improvement over more policies than existing results in the literature (\cref{sec:offline-rl}).
\end{enumerate}
In sum, we show that DistRL can yield small-loss bounds in both online and offline RL, which provide a concrete theoretical justification for the benefits of distribution learning in decision making.

\section{Related Works}
\paragraph{Theory of Distributional RL}
\citet{rowland2018analysis,rowland2023analysis} proved asymptotic convergence guarantees of popular distributional RL algorithms such as C51 \citep{bellemare2017distributional} and QR-DQN \citep{dabney2018distributional}.
However, these asymptotic results do not explain the \emph{benefits} of distributional RL over standard approaches,
since they do not imply stronger finite-sample guarantees than those obtainable with non-distributional algorithms.
In contrast, our work shows that distributional RL yields adaptive finite-sample bounds that converge faster when the optimal cost of the problem is small.
\citet{wu2023distributional} recently derived finite-sample bounds for distributional off-policy evaluation with MLE, while our offline RL section focuses on off-policy optimization.

\paragraph{First-order bounds in bandits}
When maximizing rewards, first-order ``small-return'' bounds can be easily derived from EXP4 \citep{auer2002nonstochastic}, since receiving the worst reward $0$ with probability (w.p.) $\delta$ contributes at most $R^\star\delta$ to the regret\footnote{Assume rewards/losses in $[0,1]$ and $R^\star/L^\star$ is the maximum/minimum expected reward/loss.}.
When minimizing costs, receiving the worst loss $1$ w.p. $\delta$ may induce large regret relative to $L^\star$ if $L^\star$ is small.
To illustrate, if $R^\star=0$ then all policies are optimal, so no learning is needed and the small-return bound is vacuous.
Yet if $L^\star=0$, sub-optimal policies may have a large gap from $L^\star$, so small-loss bounds in this regime are meaningful.
Small-loss bounds are achievable in multi-arm bandits \citep{foster2016learning}, semi-bandits \citep{neu2015first,lykouris2022small}, and CBs \citep{allen2018make,foster2021efficient}.

\paragraph{First-order bounds in RL}
\citet{jin2020reward,wagenmaker2022first} obtained small-return regret for tabular and linear MDPs via concentration bounds that scale with the variance. The idea is that the return's variance is bounded by some multiple of the expected value, which is bounded by $V^\star$ in the reward-maximizing setting,
\ie, $\Var(\sum_h r_h\mid \pi^k)\leq c\cdot V^{\pi^k}\leq c\cdot V^\star$.
However, the last inequality fails in the loss-minimizing setting, so the variance approach does not easily yield small-loss bounds.
Small-loss regret for tabular MDPs was
resolved by \citet[Theorem 4.1]{lee2020bias} using online mirror descent with the log-barrier on the occupancy measure.
Moreover, \citet[Theorem 3.8]{kakade2020information} obtains small-loss regret for linear-quadratic regulators (LQRs), but their Assumption 3 posits that the coefficient of variation for the cumulative costs is bounded, which is false in general even in tabular MDPs. To the best of our knowledge, there are no known first-order bounds for low-rank MDPs or in offline RL.

\paragraph{Risk-sensitive RL}
A well-motivated use-case of DistRL is risk-sensitive RL, where the goal is to learn risk-sensitive policies that optimize some risk measure, \eg, Conditional Value-at-Risk (CVaR), of the loss \citep{dabney2018distributional}.
Orthogonal to risk-sensitive RL, this work focuses on the benefits of DistRL for standard risk-neutral RL. Our insights may lead to first-order bounds for risk-sensitive RL, which we leave as future work.

\section{Preliminaries}
As warmup, we begin with the contextual bandit problem with an arbitrary context space $\Xcal$, finite action space $\Acal$ with size $A$ and conditional cost distributions $C:\Xcal\times\Acal\to\Delta([0,1])$. Throughout, we fix some dominating measure $\lambda$ on $[0,1]$ (\eg, Lebesgue for continuous or counting for discrete) and let $\Delta([0,1])$ be all distributions on $[0,1]$ that are absolutely continuous with respect to $\lambda$. We identify such a distribution with its density with respect to $\lambda$, and we also write $C(y\mid x,a)$ for $(C(x,a))(y)$.
Let $K$ denote the number of episodes. At each episode $k\in[K]$, the learner observes a context $x_k\in\Xcal$, samples an action $a_k\in\Acal$, and then receives a cost $c_t\sim C(x_t,a_t)$, which we assume to be normalized, \ie, $c_t\in[0,1]$.
The goal is to design a learner that attains low regret with high probability, where regret is defined as
\begin{equation*}\textstyle
    \op{Regret}_{\text{CB}}(K) = \sum_{k=1}^K \bar C(x_k,a_k) - \bar C(x_k,\pi^\star(x_k)),
\end{equation*}
where $\bar f = \int yf(y)\diff\lambda(y)$ for any $f\in\Delta([0,1])$ and $\pi^\star(x_k) = \argmin_{a\in\Acal}\bar C(x_k,a)$.

The focus of this paper is reinforcement learning (RL) under the Markov Decision Process (MDP) model, with observation space $\Xcal$, finite action space $\Acal$ with size $A$, horizon $H$, transition kernels $P_h:\Xcal\times\Acal\to\Delta(\Xcal)$ and \emph{cost} distributions $C_h:\Xcal\times\Acal\to\Delta([0,1])$ at each step $h\in[H]$.
We start with the \emph{Online RL} setting, which proceeds over $K$ episodes as follows:
at each episode $k\in[K]$, the learner plays a policy $\pi^k\in[\Xcal\to\Delta(\Acal)]^H$; we start from a fixed initial state $x_1$; then for each $h=1,2,\dots,H$, the policy samples an action $a_h\sim\pi^k_h(x_h)$, receives a cost $c_h\sim C_h(x_h,a_h)$, and transitions to the next state $x_{h+1}\sim P_h(x_h,a_h)$.
Our goal is to compete with the optimal policy that minimizes expected the loss, \ie, $\pi^\star\in\argmin_{\pi\in\Pi}V^\pi$ where $V^\pi = \Eb[\pi]{\sum_{h=1}^Hc_h}$.
Regret bounds aim to control the learner's regret with high probability, where regret is defined as,
\begin{equation*}\textstyle
    \op{Regret}_{\text{RL}}(K) = \sum_{k=1}^K V^{\pi^k}-V^\star.
\end{equation*}
If the algorithm returns a single policy $\wh\pi$, it is desirable to obtain a Probably Approximately Correct (PAC) bound on the sub-optimality of $\wh\pi$, \ie, $V^{\wh\pi}-V^\star$. %

The third setting we study is \emph{Offline RL}, where instead of needing to actively explore and collect data ourselves, we are given $H$ datasets $\Dcal_1,\Dcal_2,\dots,\Dcal_H$ to learn a good policy $\wh\pi$.
Each $\Dcal_h$ contains $N$ \emph{i.i.d.} samples $(x_{h,i},a_{h,i},c_{h,i},x_{h,i}')$ from the process $\prns{x_{h,i},a_{h,i}}\sim\nu_h, c_{h,i}\sim C_h(x_{h,i},a_{h,i}), x_{h,i}'\sim P_h(x_{h,i},a_{h,i})$, where $\nu_h\in\Delta(\Xcal\times\Acal)$ is arbitrary, \eg, the visitations of many policies from the current production system.
The goal is to design an offline procedure with a PAC guarantee for $\wh\pi$, which should improve over the data generating process.

\paragraph{Distributional RL}
For a policy $\pi$ and $h\in[H]$, let $Z^\pi_h(x_h,a_h)\in\Delta([0,1])$ denote the distribution of the loss-to-go $\sum_{t=h}^H c_t$ conditioned on rolling in $\pi$ from $x_h,a_h$.
The expectation of the above is $Q_h^\pi(x_h,a_h) = \bar Z^\pi_h(x_h,a_h)$ and $V_h^\pi(x_h) = \Eb[a_h\sim\pi_h(x_h)]{Q_h^\pi(x_h,a_h)}$. We use $Z^\star_h, Q^\star_h, V^\star_h$ to denote these quantities with $\pi^\star$.
Recall the regular Bellman operator acts on a function $f:\Xcal\times\Acal\to[0,1]$ as follows:
$\Tcal_h^\pi f(x,a) = \bar C_h(x,a) + \Eb[x'\sim P_h(x,a),a'\sim\pi(x')]{f(x',a')}.$
Analogously, the distributional Bellman operator \citep{morimura2012parametric,bellemare2017distributional} acts on a conditional distribution $d:\Xcal\times\Acal\to\Delta([0,1])$ as follows: $\Tcal_h^{\pi,D} d(x,a) \stackrel{D}{=} C_h(x,a) + d(x',a')$, where $x'\sim P_h(x,a), a'\sim \pi(x')$ and $\stackrel{D}{=}$ denotes equality of distributions. Another way to think about the distributional Bellman operator is that a sample $z\sim \Tcal_h^{\pi, D} d(x,a)$ is generated as follow: %
$z:= c + y, \text{ where } c \sim C_h(x,a), x'\sim P_h(x,a), a'\sim \pi(x'), y \sim d(x',a').$
We will also use the Bellman optimality operator $\Tcal_h^\star$ and its distributional variant $\Tcal_h^{\star,D}$, defined as follows:
$\Tcal_h^\star f(x,a) = \bar C_h(x,a) + \Eb[x'\sim P_h(x,a)]{\min_{a\in\Acal}f(x',a')}$ and $\Tcal_h^{\star,D}d(x,a)\stackrel{D}{=} C_h(x,a)+d(x',a')$ where $x'\sim P_h(x,a), a'=\argmin_a\bar d(x',a)$.
Please see \cref{tab:notation} for an index of notations.

\section{Warm up: Small-Loss Regret for Distributional Contextual Bandits}\label{sec:warm-up}
In this section, we propose an efficient reduction from CB to online maximum likelihood estimation (MLE),
which is the standard tool for distribution learning that we will use throughout the paper.
In our CB algorithm, we balance exploration and exploitation with the reweighted inverse gap weighting (ReIGW) of \citet{foster2021efficient},
which defines a distribution over actions given predictions $\wh f\in\RR^A$ and a parameter $\gamma\in\RR_{++}$:
setting $b = \argmin_{a\in\Acal}\wh f(a)$ as the best action with respect to the predictions, the weight for any other action $a\neq b$ is,
\begin{equation}
    \text{ReIGW}_\gamma(\wh f,\gamma)[a] := \frac{\wh f(b)}{A\wh f(b) + \gamma\prns*{\wh f(a)-\wh f(b)}}, \label{eq:igw}
\end{equation}
and the rest of the weight is allocated to $b$: $\text{ReIGW}_\gamma(\wh f,\gamma)[b] = 1-\sum_{a\neq b}\text{ReIGW}_\gamma(\wh f,\gamma)[a]$.

\begin{algorithm}%
    \caption{Distributional CB (\cb{})}
    \label{alg:distcb}
    \begin{algorithmic}[1]
        \State\textbf{Input:} number of episodes $K$, failure probability $\delta$, ReIGW learning rate $\gamma$.
        \State Initialize any cost distribution $f^{(1)}$.
        \For{episode $k=1,2,\dots,K$}
            \State Observe context $x_k$.
            \State Sample action $a_k\sim p_k = \text{ReIGW}(\bar f^{(k)}(x_k,\cdot),\gamma)$ from \cref{eq:igw}. \label{line:distcb-reigw}
            \State Observe cost $c_k\sim C(x_k,a_k)$ and update online MLE oracle with $((x_k,a_k), c_k)$. \label{line:distcb-learn-f}
        \EndFor
    \end{algorithmic}
\end{algorithm}

We propose \tb{Dist}ributional \tb{C}ontextual \tb{B}andit (\cb{}) in \cref{alg:distcb}, a two-step procedure for each episode $k\in[K]$. Upon seeing context $x_k$, \cb{} first samples an action $a_k$ from ReIGW generated by means of our estimated cost distributions for each action, \ie, $\wh f(a) = \bar f^{(k)}(x_k,a), \forall a\in\Acal$ (\cref{line:distcb-reigw}).
Then, \cb{} updates $f^{(k)}(\cdot\mid x_k,a_k)$ by maximizing the log-likelihood to estimate the conditional cost distribution $C(\cdot\mid x_k,a_k)$ (\cref{line:distcb-learn-f}).
Formally, this second step is achieved via an online MLE oracle with a realizable distribution class $\Fcal_{CB}\subset \Xcal\times\Acal\to\Delta([0,1])$; let $\op{Regret}_{\log}(K)$ be some upper bound on the log-likelihood regret for all possibly adaptive sequences $\braces{x_k,a_k,c_k}_{k\in[K]}$,
\begin{equation*}\textstyle
     \sum_{k=1}^K \log C(c_k\mid x_k,a_k)-\log f^{(k)}(c_k\mid x_k,a_k)\leq \op{Regret}_{\log}(K).
\end{equation*}
Under \emph{realizability}, $C\in\Fcal_{CB}$, we expect $\op{Regret}_{\log}(K)\in\Ocal\prns{\log(K)}$.
For instance, if $\Fcal_{CB}$ is finite, exponentially weighted average forecaster guarantees $\op{Regret}_{\log}(K)\leq\log|\Fcal_{CB}|$ \citep[Chapter 9]{cesa2006prediction}.
We now state our main result for \cb{}.
\begin{restatable}{theorem}{cbRegret}\label{thm:distcb}
For any $\delta\in(0,1)$, w.p. at least $1-\delta$, running \cb{} with $\gamma = 10A\vee \sqrt{\frac{40A\prns{C^\star+\log(1/\delta)}}{112\prns{\op{Regret}_{\log}(K)+\log(1/\delta)}}}$ has regret scaling with $C^\star = \sum_{k=1}^K \min_{a\in\Acal}\bar C(x_k,a)$,
\begin{equation*}
    \op{Regret}_{\cb{}}(K) \leq 232\sqrt{AC^\star\op{Regret}_{\log}(K)\log(1/\delta)} + 2300A\prns{\op{Regret}_{\log}(K)+\log(1/\delta)}.
\end{equation*}
\end{restatable}
The dominant term scales with the optimal sum of costs $\sqrt{C^\star}$ which shows that \cb{} obtains small-loss regret. \cb{} is also computationally efficient since each episode simply requires computing the ReIGW.
FastCB is the only other computationally efficient CB algorithm with small-loss regret \citep[Theorem 1]{foster2021efficient}.
Our bound matches that of FastCB in terms of dependence on $A,C^\star$ and $\log(1/\delta)$.
Our key difference with FastCB is the online supervised learning oracle:
in \cb{}, we aim to learn the conditional cost distribution by maximizing log-likelihood,
while FastCB aims to perform regression with the binary cross-entropy loss.
In \cref{sec:experiments}, we find that \cb{} empirically outperforms  SquareCB and FastCB in three challenging CB tasks, which reinforces the practical benefits of distribution learning in CB setting.

\subsection{Proof Sketch}
First, apply the per-round inequality for ReIGW \citep[Theorem 4]{foster2021efficient} to get,
\begin{equation*}
    \op{Regret}_{\text{DistCB}}(K)\lesssim \sum_{k=1}^K \EE_{a_k\sim p_k}\Biggl[\frac{A}{\gamma} \bar C(s_k,a_k) + \gamma \underbrace{\frac{\prns{\bar f^{(k)}(s_k,a_k)-\bar C(s_k,a_k)}^2}{\bar f^{(k)}(s_k,a_k)+\bar C(s_k,a_k)}}_{\bigstar} \Biggr].
\end{equation*}
For any distributions $f,g\in\Delta([0,1])$, their triangular
discrimination\footnote{Triangular discrimination is also known as Vincze-Le Cam divergence \citep{vincze1981concept,le2012asymptotic}.} is defined as $D_\triangle(f\Mid g) := \int \frac{\prns{f(y)-g(y)}^2}{f(y)+g(y)}\diff\lambda(y)$.
The key insight is that $\bigstar$ can be bounded by the triangular discrimination of $f^{(k)}(s_k,a_k)$ and $C(s_k,a_k)$:
by Cauchy-Schwartz and $y^2\leq y$ for $y\in[0,1$], we have $\bar f-\bar g= \int y\prns{ f(y)-g(y) }\diff\lambda(y) \leq \sqrt{ \int y \prns{f(y)+g(y)}\diff\lambda(y) } \sqrt{ \int \frac{\prns{f(y)-g(y)}^2}{f(y)+g(y)}\diff\lambda(y) }$, and hence,
\begin{equation}
    \abs{\bar f-\bar g} \leq \sqrt{ \prns{ \bar f + \bar g } D_\triangle\prns{ f\Mid g } }. \tag{$\triangle_1$} \label{eq:tri-disc-ineq-1}
\end{equation}
So, \cref{eq:tri-disc-ineq-1} implies that $\bigstar$ is bounded by $D_\triangle(f^{(k)}(s_k,a_k)\Mid C(s_k,a_k))$. Since $D_\triangle$ is equivalent (up to universal constants) to the squared Hellinger distance, \citet[Lemma A.14]{foster2021statistical} implies the above can be bounded by the online MLE regret, so w.p. at least $1-\delta$, we have
\begin{equation*}\textstyle
    \op{Regret}_{\text{DistCB}}(K)\lesssim \sum_{k=1}^K \frac{A}{\gamma} \prns{\bar C(s_k,a_k)+\log(1/\delta)} + \gamma\prns{ \op{Regret}_{\log}(K) + \log(1/\delta) }.
\end{equation*}
From here, we just need to rearrange terms and set the correct $\gamma$.
\cref{app:proofs-distcb} contains the full proof.

\section{Small-Loss Bounds for Online Distributional RL}\label{sec:online-rl}

\begin{algorithm}[t!]
\caption{\tb{O}ptimistic \tb{Dis}tributional \tb{C}onfidence set \tb{O}ptimization (\online{})}
\label{alg:onlinerl}
\begin{algorithmic}[1]
    \State\textbf{Input:} number of episodes $K$, distribution class $\Fcal$, threshold $\beta$.
    \State Initialize $\Dcal_{h,0}\gets\emptyset$ for all $h\in[H]$, and set $\Fcal_0 = \Fcal$.
    \For{episode $k=1,2,\dots,K$}
        \State Set optimistic estimate $f^{(k)} = \argmin_{f\in\Fcal_{k-1}} \min_a \bar f_1(x_1,a)$. \label{line:distgolf-argmin-fk}
        \State Set $\pi^k_h(x) = \argmin_a\bar f^{(k)}_{h}(x,a)$.
        \State \multiline{Roll out $\pi^k$ and obtain a trajectory $x_{1,k},a_{1,k},c_{1,k},\dots,x_{H,k},a_{H,k},c_{H,k}$. \\ For each $h\in[H]$, augment the dataset $\Dcal_{h,k} = \Dcal_{h,k-1}\cup\braces{(x_{h,k},a_{h,k},c_{h,k},x_{h+1,k})}$. }\label{line:distgolf-execute}
        \State \multiline{For all $(h,f)\in[H]\times\Fcal$, sample $y_{h,i}^f \sim f_{h+1}(x_{h,i}',a')$ and $a'=\argmin_a \bar f_{h+1}(x_{h,i}',a)$, where $(x_{h,i},a_{h,i},c_{h,i},x_{h,i}')$ is the $i$-th datapoint of $\Dcal_{h,k}$.
        Then, set $z_{h,i}^f=c_{h,i}+y_{h,i}^f$ and define the confidence set} \label{line:distgolf-confidence-set}
        \begin{equation*}
            \Fcal_k = \braces{ f\in\Fcal: \sum_{i=1}^k\log f_h(z_{h,i}^f\mid x_{h,i},a_{h,i})\geq \max_{g\in\Fcal_h}\sum_{i=1}^k\log g(z_{h,i}^f\mid x_{h,i},a_{h,i})-7\beta, \forall h\in[H] }.
        \end{equation*}
    \EndFor
    \State \textbf{Output:} $\bar\pi = \op{unif}(\pi^{1:K})$.
\end{algorithmic}
\end{algorithm}

We now extend our insights to the online RL setting and propose a DistRL perspective on GOLF \citep{jin2021bellman}.
While GOLF constructs confidence sets of near-minimizers of the squared Bellman error loss, we propose to construct these confidence sets using near-maximizers of the log-likelihood loss to approximate MLE.
To leverage function approximation for learning conditional distributions, we use a generic function class $\Fcal\subseteq(\Xcal\times\Acal\to\Delta([0,1]))^H$ where each element $f\in\Fcal$ is a tuple $f=(f_1,\dots,f_H)$ such that each $f_h$ is a candidate estimator for $Z^\star_h$, the distribution of loss-to-go $\sum_{t=h}^Hc_t$ under $\pi^\star$. For notation, $f_{H+1}(x,a)=\delta_0$ denotes the dirac at zero for all $x,a$.

We now present our \tb{O}ptimistic \tb{Dis}tributional \tb{C}onfidence set \tb{O}ptimization (\online{}) algorithm in \cref{alg:onlinerl}, consisting of three key steps per episode.
At episode $k\in[K]$, \online{} first identifies the $f^{(k)}$ with the minimal expected value at $h=1$ over the previous confidence set $\Fcal_{k-1}$ (\cref{line:distgolf-argmin-fk}). This step induces \emph{global optimism}.
Then, \online{} collects data for this episode by rolling in with the greedy policy $\pi^k$ with respect to the mean of $f^{(k)}$ (\cref{line:distgolf-execute}).
Finally, \online{} constructs a confidence set $\Fcal_k$ by including a function $f$ if it exceeds a threshold on the log-likelihood objective using data $z_{h,i}^f\sim \Tcal_h^{\star,D}f_{h+1}(x_{h,i},a_{h,i})$ for all steps $h$ simultaneously (\cref{line:distgolf-confidence-set}). This step is called \emph{local fitting}, as each $f\in\Fcal_k$ has the property that $f_h$ is close-in-distribution to $\Tcal_h^{\star, D}f_{h+1}$ for all $h$.
We highlight that \online{} only learns the distribution for estimating the mean, \ie, \cref{line:distgolf-argmin-fk,line:distgolf-execute} only use the mean $\bar f$.
This seemingly roundabout way of estimating the mean is exactly how distributional RL algorithms such as C51 differ from the classic DQN.

To ensure that MLE succeeds for the Temporal-Difference (TD) style confidence sets, we need the following distributional Bellman Completeness (BC) condition introduced in \citet{wu2023distributional}.
\begin{restatable}[Bellman Completeness]{assumption}{policyDistBC}\label{ass:policy-dist-bellman-completeness}
For all $\pi,h\in[H]$, $f_{h+1}\in\Fcal_{h+1}\implies \Tcal^{\pi,D}_hf_{h+1}\in\Fcal_h$.
\end{restatable}

\subsection{The $\ell_1$ Distributional Eluder Dimension}
We now introduce the $\ell_1$ distributional eluder dimension. Let $\Scal$ be an abstract input space, let $\Psi$ be a set of functions mapping $\Scal\to\RR$ and let $\Dcal$ be a set of distributions on $\Scal$.
\begin{definition}[$\ell_p$-distributional eluder dimension]
For any function class $\Psi\subseteq\Scal\to\RR$, distribution class $\Dcal\subseteq\Delta(\Scal)$ and $\eps > 0$, the $\ell_p$-distributional eluder dimension (denoted by $\DE_p(\Psi,\Dcal,\eps)$) is the length $L$ of the longest sequence $d^{(1)},d^{(2)},\dots,d^{(L)}\subseteq\Dcal$ such that there exists $\eps'\geq\eps$, such that for all $t\in[L]$, we have that there exists $f\in\Psi$ such that $\abs{\EE_{d^{(t)}} f}>\eps$ and also $\sum_{i=1}^{t-1}\abs{\EE_{d^{(i)}}f}^p\leq \eps^p$.
\end{definition}
When $p=2$, this is exactly the $\ell_2$ distributional eluder of \citet[Definition 7]{jin2021bellman}. We're particularly interested in the $p=1$ case, which can be used with MLE's generalization bounds. The following is a key pigeonhole principle for the $\ell_1$ distributional eluder dimension.
\begin{restatable}{theorem}{eluderPigeonhole}\label{thm:eluder-pigeonhole}
Let $C := \sup_{d\in\Dcal,f\in\Psi}\abs{\EE_d f}$ be the envelope.
Fix any $K\in\NN$ and sequences $f^{(1)},\dots,f^{(K)}\subseteq \Psi$, $d^{(1)},\dots,d^{(K)}\subseteq\Dcal$.
Let $\beta$ be a constant such that for all $k\in[K]$, we have,
$
    \sum_{i=1}^{k-1}\abs{\EE_{d^{(i)}}f^{(k)}} \leq \beta.
$
Then, for all $k\in[K]$, we have
\begin{equation*}
    \sum_{t=1}^k\abs{\EE_{d^{(t)}}f^{(t)}}\leq \inf_{0<\eps\leq 1}\braces{ \DE_1(\Psi,\Dcal,\eps)(2C + \beta\log(C/\eps)) + k\eps }.
\end{equation*}
\end{restatable}
As we'll see later, \cref{thm:eluder-pigeonhole} is the key tool that transfers triangular discrimination guarantees on the training distribution to any new test distribution.
Another key property is that the $\ell_1$ dimension generalizes the original $\ell_2$ dimension of \citet{jin2021bellman}.
\begin{restatable}{lemma}{eluderOneGeneralizesTwo}\label{lem:eluder-one-generalizes-two}
For any $\Psi,\Dcal$ and $\eps>0$, we have $\DE_1(\Psi,\Dcal,\eps)\leq\DE_2(\Psi,\Dcal,\eps)$.
\end{restatable}
Finally, we note that our distributional eluder dimension generalize the regular $\ell_1$ eluder from \citet{liu2022partially}, which can be seen by taking $\Dcal$ to be dirac distributions.

\subsection{Small-Loss Bounds for \online{}}
We will soon prove small-loss regret bounds with the ``Q-type'' dimension, where ``Q-type'' refers to the fact that $\Scal=\Xcal\times\Acal$. While low-rank MDPs are not captured by the ``Q-type'' dimension, they are captured by the ``V-type'' dimension where $\Scal=\Xcal$ \citep{jin2021bellman,du2021bilinear}. For PAC bounds with the V-type dimension, we need to slightly modify the data collection process in \cref{line:distgolf-execute} with \emph{uniform action exploration} (UAE). Instead of executing $\pi^k$ for a single trajectory, partially roll-out $\pi^k$ for $H$ times where for each $h\in[H]$, we collect $x_{h,k}\sim d^{\pi^k}_h$, take a random action $a_{h,k}\sim \op{unif}(\Acal)$, observe $c_{h,k}\sim C_h(x_{h,k},a_{h,k}),x_{h,k}'\sim P_h(x_{h,k},a_{h,k})$ and augment the dataset $\Dcal_{h,k}=\Dcal_{h,k-1}\cup\{(x_{h,k},a_{h,k},c_{h,k},x_{h,k}')\}$. The modified algorithm is detailed in \cref{app:omitted-algs}.

We lastly need to define the function and distribution classes measured by the distributional eluder dimension. The Q-type classes are $\Dcal_h = \braces{(x,a)\mapsto d^\pi_h(x,a): \pi\in\Pi}$ and $\Psi_h=\braces{(x,a)\mapsto D_\triangle(f(x,a)\Mid \Tcal^{\star,D}f(x,a)): f\in\Fcal}$. Similarly, the V-type classes are $\Dcal_{h,v}=\braces{x\mapsto d^\pi_h(x):\pi\in\Pi}$ and $\Phi_{h,v}=\braces{ x\mapsto \EE_{a\sim \op{Unif}(\Acal)}[D_\triangle(f(x,a)\Mid \Tcal^{\star,D}f(x,a))] : f\in\Fcal }$.
Finally, define $\DE_1(\eps) = \max_h \DE_1(\Psi_h,\Dcal_h,\eps)$ and $\DE_{1,v}(\eps) = \max_h \DE_1(\Psi_{h,v},\Dcal_{h,v},\eps)$.
\begin{restatable}{theorem}{onlineRLGeneral}\label{thm:online-general}
Suppose DistBC holds (\cref{ass:policy-dist-bellman-completeness}). For any $\delta\in(0,1)$, w.p. at least $1-\delta$, running \online{} with $\beta = \log(HK|\Fcal|/\delta)$ guarantees the following regret bound,
\begin{align*}
\op{Regret}_{\online{}}(K)\leq 160H\sqrt{KV^\star \DE_1(1/K)\log(K) \beta} + 18000H^2\DE_1(1/K)\log(K)\beta.
\end{align*}
If $\textsc{UAE}=\textsc{True}$ (\cref{alg:onlinerl_appendix}), then the learned mixture policy $\bar\pi$ is guaranteed to satisfy,
\begin{align*}
    V^{\bar\pi}-V^\star\leq 160H\sqrt{\frac{AV^\star \DE_{1,v}(1/K)\log(K) \beta}{K}} + \frac{18000H^2A\DE_{1,v}(1/K)\log(K)\beta}{K}.
\end{align*}
\end{restatable}
Compared to prior bounds for GOLF \citep{jin2021bellman}, the leading $\sqrt{K}$ terms in our bounds enjoy the same sharp dependence in $H,K$ and the eluder dimension. Our bounds further enjoy one key improvement: the leading terms are multiplied with the instance-dependent optimal cost $V^\star$, giving our bounds the \emph{small-loss} property.
For example, if $V^\star\leq\Ocal(1/\sqrt{K})$, then our regret bound converges at a fast $\Ocal(H^2\DE_1(1/K)\log(K)\beta)$ rate.
While there are existing first-order bounds in online RL, our bound significantly improves on their generality.
For example, \citet{zanette2019tighter,jin2020reward,wagenmaker2022first} used Bernstein bonuses that scale with the conditional variance and showed that careful analysis can lead to ``small-return'' bounds in tabular and linear MDPs.
However, ``small-return'' bounds do not imply ``small-loss'' bounds and ``small-loss'' bounds are often harder to obtain\footnote{In \cref{sec:small-reward}, we show a slight modification of our approach also yields ``small-return'' bounds.}. While it is possible that surgical analysis with variance bonuses can lead to small-loss bounds in tabular and linear MDPs, this approach may not scale to settings with non-linear function approximation such as low-rank MDPs.

\paragraph{On Bellman Completeness}
Exponential error amplification can occur in online and offline RL under only realizability of $Q$ functions \citep{wang2021what,wang2021instabilities,wang2021exponential,foster2022offline}. With only realizability, basic algorithms such as TD and Fitted-$Q$-Evaluation (FQE) can diverge or converge to bad fixed point solutions  \citep{tsitsiklis1996analysis,munos2008finite,kolter2011fixed}. As a result, BC has risen as a \emph{de facto} sufficient condition for sample efficient RL \citep{chang2022learning,xie2021bellman,zanette2021provable}.
Finally, we highlight that our method can be easily extended to hold under \emph{generalized completeness}, \ie, there exist function classes $\Gcal_h$ such that $f_{h+1}\in\Fcal_{h+1}\implies \Tcal_h^{\pi,D}f_{h+1}\in\Gcal_h$ \citep[as in][Assumption 14]{jin2021bellman}. Simply replace $\max_{g\in\Fcal_h}$ in the confidence set construction with $\max_{g\in\Gcal_h}$. While adding functions to $\Fcal$ may break BC (as BC is not monotonic), we can always augment $\Gcal$ to satisfy generalized completeness.

\paragraph{Computational complexity}
When taken as is, OLIVE \citep{jiang2017contextual}, GOLF, and our algorithms are version space methods that suffer from a computational drawback: optimizing over the confidence set is NP-hard \citep{dann2018oracle}.
However, the confidence set is purely for deep exploration via optimism and can be replaced by other computationally efficient exploration strategies.
For example, $\eps$-greedy suffices in problems that don't require deep and strategic exploration, \ie, a large myopic exploration gap \citep{dann2022guarantees}. With $\eps$-greedy, a replay buffer, and discretization, our algorithm essentially recovers C51 \citep{bellemare2017distributional}.
We leave developing and analyzing computationally efficient algorithms based on our insights as promising future work.

\subsection{Instantiation with Low-Rank MDPs}
The low-rank MDP \citep{agarwal2020flambe} is a standard abstraction for non-linear function approximation used in theory \citep{uehara2021representation} and practice \citep{zhang2022making,chang2022learning}.
\begin{restatable}[Low-rank MDP]{definition}{lowRankMDPDef}\label{def:low-rank-mdp}
A transition model $P_h:\Xcal\times\Acal\to\Delta(\Xcal)$ has rank $d$
if there exist unknown features $\phi_h^\star:\Xcal\times\Acal\to\RR^d,\mu_h^\star:\Xcal\to\RR^d$ such that
$P_h(x'\mid x,a) = \phi_h^\star(x,a)^\top \mu_h^\star(x')$ for all $x,a,x'$.
Also, assume $\max_{x,a}\|\phi_h^\star(x,a)\|_2\leq 1$ and $\|\int g\diff\mu_h^\star\|_2\leq\|g\|_\infty \sqrt{d}$ for all functions $g:\Xcal\to\RR$.
The MDP is called low-rank if $P_h$ is low-rank for all $h\in[H]$.
\end{restatable}

We now specialize \cref{thm:online-general} to low-rank MDPs with three key steps.
First, we bound the V-type eluder dimension by $\DE_{1,v}(\eps)\leq\Ocal(d\log(d/\eps))$, which is a known result that we reproduce in \cref{thm:low-rank-mdp-eluder}.
The next step requires access to a realizable $\Phi$ class, \ie, for all $h\in[H]$, $\phi^\star_h\in\Phi$, which is a standard assumption for low-rank MDPs \citep{agarwal2020flambe,uehara2021representation,mhammedi2023efficient}.
Given the realizable $\Phi$, we can construct a specialized $\Fcal$ for the low-rank MDP: $\Fcal^{\op{lin}}=\Fcal_1^{\op{lin}}\times\dots\times\Fcal_H^{\op{lin}}\times\Fcal_{H+1}^{\op{lin}}$ where $\Fcal_{H+1}^{\op{lin}}=\braces{\delta_0}$ and for all $h\in[H]$,
\begin{align}\label{eq:linear-mdp-f-class}
    \mathcal{F}_h^{\text{lin}}=\bigg\{&
	f(z\mid x,a)=\big\langle\phi(x,a),w(z)\big\rangle
	\quad:\quad \phi\in\Phi, w:[0,1]\rightarrow\mathbb{R}^d, \\
	&\text{\; s.t.\; }
	\max_z \|w(z)\|_2\leq \alpha\sqrt{d}
	\text{\; and\; }
\max_{x,a,z}\big\langle\phi(x,a),w(z)\big\rangle\leq \alpha
	\bigg\}, \nonumber
\end{align}
where $\alpha:=\max_{h,\pi,z,x,a}Z^\pi_h(z\mid x,a)$ is the largest mass for the cost-to-go distributions. In \cref{sec:dist-bc-low-rank-mdps}, we show that $\Fcal^{\op{lin}}$ satisfies DistBC. Further, if costs are discretized into a uniform grid of $M$ points, its bracketing entropy is bounded by $\wt\Ocal\prns{dM + \log|\Phi|}$. Discretization is necessary to bound the statistical complexity of $\Fcal^{\op{lin}}$ and is common in practice, \eg, C51 and Rainbow both set $M=51$ which works well in Atari games \citep{bellemare2017distributional,hessel2018rainbow}.
\begin{theorem}\label{thm:online-lnnr-mdp-pac}
Suppose the MDP is low-rank. For any $\delta\in(0,1)$, w.p. at least $1-\delta$, running \online{} with \textsc{UAE}=\textsc{True} and with $\Fcal^{\op{lin}}$ as described above learns a policy $\bar\pi$ such that,
\begin{equation*}
    V^{\bar\pi}-V^\star\in\wt\Ocal\prns{ H\sqrt{\frac{Ad V^\star\prns{dM+\log(|\Phi|/\delta)}}{K}} + \frac{AdH^2\prns{dM+\log(|\Phi|/\delta)}}{K} }.
\end{equation*}
\end{theorem}
\begin{proof}
As described above, we have $\DE_1(1/K)\leq\Ocal(d\log(dK))$ and $\beta=\log(HK/\delta)+dM + \log|\Phi|$. Since DistBC is satisfied by $\Fcal^{\op{lin}}$, plugging into \cref{thm:online-general} gives the result.
\end{proof}
This is the first small-loss bound for low-rank MDPs, and for online RL with non-linear function approximation in general.
Again when $V^\star\leq\wt\Ocal(1/K)$, \online{} has a fast $\wt\Ocal(1/K)$ convergence rate which improves over all prior results that converge at a slow $\wt\Omega(1/\sqrt{K})$ rate \citep{uehara2021representation}.

\subsection{Proof Sketch of \cref{thm:online-general}}\label{sec:proof-sketch-online}
By DistBC (\cref{ass:policy-dist-bellman-completeness}),
we can deduce two facts about the construction of $\Fcal_k$: (i) $Z^\star\in\Fcal_k$, and
(ii) elements of $\Fcal_k$ almost satisfy the distributional Bellman equation, \ie, for all $h\in[H]$, we have $\sum_{i=1}^{k}\Eb[\pi^i]{\delta_{h,k}(x_h,a_h)}\leq \Ocal(\beta)$ where $\delta_{h,k}(x_h,a_h) = D_\triangle(f_h^{(k)}(x_h,a_h)\Mid\Tcal_h^{\star,D}f_{h+1}^{(k)}(x_h,a_h))$.
Next, we derive a corollary of \cref{eq:tri-disc-ineq-1}:
\begin{equation}
    \abs{\bar f-\bar g}\leq \sqrt{4\bar{g} + D_\triangle(f\Mid g)}\cdot \sqrt{D_\triangle(f\Mid g)}. \tag{$\triangle_2$}\label{eq:tri-disc-ineq-2}
\end{equation}
To see why this is true, apply AM-GM to \cref{eq:tri-disc-ineq-1} to get $2(\bar f-\bar g)\leq \bar f+\bar g + D_\triangle(f\Mid g)$, which simplifies to $\bar f\leq 3\bar g + D_\triangle(f\Mid g)$. Plugging this back into \cref{eq:tri-disc-ineq-1} yields \cref{eq:tri-disc-ineq-2}.
Then, by iterating \cref{eq:tri-disc-ineq-2} and AM-GM, we derive a self-bounding lemma: for any $f,\pi,h$, we have $\bar f_h(x_h,a_h)\lesssim Q^{\pi}_h(x_h,a_h)+H\sum_{t=h}^H\EE_{\pi,x_h,a_h}[D_\triangle(f_t(x_t,a_t)\Mid \Tcal_h^{\pi,D}f_{t+1}(x_t,a_t))]$ (\cref{lem:self-bounding}).
Since $\Tcal_h^{\pi^k}\bar f_{h+1}^{(k)}(x,a)=\overline{\Tcal_h^{\pi^k,D} f_{h+1}^{(k)}(x,a)}$ and $\Tcal_h^{\pi^k,D}f_{h+1}^{(k)} = \Tcal_h^{\star,D}f_{h+1}^{(k)}$, we have
\begin{align*}
    \textstyle V^{\pi^k}-V^\star&\textstyle \leq V^{\pi^k}-\bar f_1^{(k)}(x_1,\pi^k_1(x_1))  \tag{optimism from fact (i)}
    \\&\textstyle=\sum_{h=1}^H\Eb[\pi^k]{\Tcal_h^{\pi^k}\bar f_{h+1}^{(k)}(x_h,a_h)-\bar f_h^{(k)}(x_h,a_h)} \tag{performance difference}
    \\&\textstyle\leq2\sum_{h=1}^H\sqrt{\EE_{\pi^k}\bracks*{\bar f_h^{(k)}(x_h,a_h) + \delta_{h,k}(x_h,a_h)}}\sqrt{\Eb[\pi^k]{\delta_{h,k}(x_h,a_h)}} \tag{\cref{eq:tri-disc-ineq-2}}
    \\&\textstyle \lesssim \sqrt{ V^{\pi^k} w+ H\sum_{h=1}^H\Eb[\pi^k]{\delta_{h,k}(x_h,a_h)} }\sqrt{H\Eb[\pi^k]{\delta_{h,k}(x_h,a_h)}}. \tag{\cref{lem:self-bounding}}
\end{align*}
The implicit inequality $V^{\pi^k}-V^\star\lesssim \sqrt{ V^{\star} + H\sum_{h=1}^H\EE_{\pi^k}\bracks*{\delta_{h,k}(x_h,a_h)} }\sqrt{H\Eb[\pi^k]{\delta_{h,k}(x_h,a_h)}}$ can then be obtained by AM-GM and rearranging.
The final step is to sum over $k$ and bound $\sum_{k=1}^K\Eb[\pi^k]{\delta_{h,k}(x_h,a_h)}$ via the eluder dimension's pigeonhole principle (\cref{thm:eluder-pigeonhole} applied with fact (ii)). Please see \cref{app:proofs-online-rl} for the full proof.

\section{Small-Loss Bounds for Offline Distributional RL}\label{sec:offline-rl}
We now propose \tb{P}essimistic \tb{Dis}tributional \tb{C}onfidence set \tb{O}ptimization (\offline{}; \cref{alg:offline}), which adapts the distributional confidence set technique from the previous section to the offline setting by leveraging pessimism instead of optimism.
Notably, \offline{} is a simple two-step algorithm that achieves the first small-loss PAC bounds in offline RL.
First, construct a distributional confidence set for each policy $\pi$ based on a similar log-likelihood thresholding procedure as in \online{}, where the difference is we now use data sampled from $\Tcal_h^{\pi,D}f_{h+1}$ instead of $\Tcal_h^{\star,D}f_{h+1}$.
Next, output the policy with the most pessimistic mean amongst all the confidence sets.

\begin{algorithm}%
\caption{\tb{P}essimistic \tb{Dis}tributional \tb{C}onfidence set \tb{O}ptimization (\offline{})}
\label{alg:offline}
\begin{algorithmic}[1]
    \State\textbf{Input:} datasets $\Dcal_1,\dots,\Dcal_H$, distribution function class $\Fcal$, threshold $\beta$, policy class $\Pi$.
    \State \multiline{For all $(h,f,\pi)\in[H]\times\Fcal\times\Pi$, sample $y_{h,i}^{f,\pi} \sim f_{h+1}(x_{h,i}',\pi_{h+1}(x_{h,i}'))$, where $(x_{h,i},a_{h,i},c_{h,i},x_{h,i}')$ is the $i$-th datapoint of $\Dcal_h$.
    Then, set $z_{h,i}^{f,\pi} = c_{h,i} + y_{h,i}^{f,\pi}$ and define the confidence set,} \label{line:distbco-confidence-set}
    \begin{equation*}\hspace{-0.4cm}
        \Fcal_\pi = \braces{ f\in\Fcal: \sum_{i=1}^{N}\log f_{h}(z_{h,i}^{f,\pi}\mid x_{h,i},a_{h,i}) \geq \max_{g\in\Fcal_h}\sum_{i=1}^{N}\log g(z_{h,i}^{f,\pi}\mid x_{h,i},a_{h,i})-7\beta, \forall h\in[H] }.
    \end{equation*}
    \State For each $\pi\in\Pi$, define the pessimistic estimate $f^\pi = \argmax_{f\in\Fcal_\pi}\Eb[a\sim\pi(x_1)]{\bar f_1(x_1,a)}$.
    \State \textbf{Output:} $\wh\pi = \argmax_{\pi\in\Pi}\Eb[a\sim\pi(x_1)]{\bar f^\pi_1(x_1,\pi)}$. \label{line:distbco-policy-selection}
    \end{algorithmic}
\end{algorithm}

In offline RL, many works made strong all-policy coverage assumptions \citep{antos2008learning,chen2019information}.
Recent advancements \citep{kidambi2020morel,xie2021bellman,uehara2022pessimistic,rashidinejad2021bridging,jin2021pessimism} have pursued \emph{best effort} guarantees that aim to compete with any covered policy $\wt\pi$, with sub-optimality of the learned $\wh\pi$ degrading gracefully as coverage worsens. The coverage is measured by the single-policy concentrability
$
    C^{\wt\pi} = \max_h\nm{\nicefrac{\diff d^{\wt\pi}_h}{\diff\nu_h}}_\infty.
$
We adopt this framework and obtain the first small-loss PAC bound in offline RL.

\begin{restatable}[Small-Loss PAC bound for \offline{}]{theorem}{offlinePAC}\label{thm:offline}
Assume \cref{ass:policy-dist-bellman-completeness}.
For any $\delta\in(0,1)$, w.p. at least $1-\delta$, running \offline{} with $\beta=\log(H|\Pi||\Fcal|/\delta)$ learns a policy $\wh\pi$ that enjoys the following PAC bound with respect to any comparator policy $\wt\pi\in\Pi$:
\begin{equation*}
    V^{\wh\pi}-V^{\wt\pi}\leq 9H\sqrt{\frac{C^{\wt\pi} V^{\wt\pi}\beta}{N} } + \frac{30H^2C^{\wt\pi}\beta}{N}.
\end{equation*}
\end{restatable}
To the best of our knowledge, this is the first small-loss bound for offline RL, which we highlight illustrates a novel robustness property against bad coverage. Namely, the dominant term not only scales with the coverage coefficient $C^{\wt\pi}$ but also the comparator policy's value $V^{\wt\pi}$.
In particular, \offline{} can strongly compete with a comparator policy $\wt\pi$ if \emph{one of the following} is true:
(i) $\nu$ has good coverage over $\wt\pi$, so the $\Ocal(1/\sqrt{N})$ term is manageable;
\emph{or} (ii) $\wt\pi$ has small-loss, in which case we may even obtain a fast $\Ocal(1/N)$ rate.
Thus, \offline{} has \emph{two} chances at strongly competing with $\wt\pi$, while conventional offline RL methods solely rely on (i) to be true.

\section{Distributional CB Experiments}\label{sec:experiments}
\setlength{\columnsep}{0.6cm}%
\begin{wraptable}{r}{6.9cm}
\vspace{-0.25cm}
\begin{center}
\label{table:results}
\centering
\tabcolsep=1pt
\raisebox{0pt}[\dimexpr\height-3\baselineskip\relax]{%
\begin{small}
\begin{tabular}{lccc}
\toprule
Algorithm: & SquareCB & FastCB & DistCB (Ours) \\
\midrule
\multicolumn{4}{l}{King County Housing \citep{OpenML2013}} \\
\midrule
All episodes      & .756 (.0007) & .734 (.0007) & \tb{.726} (.0003) \\
Last 100 ep. & .725 (.0012) & .719 (.0013) & \tb{.708} (.0019)  \\
\midrule
\multicolumn{4}{l}{Prudential Life Insurance \citep{prudential-life-insurance-assessment}} \\
\midrule
All episodes    & .456 (.0082)  & .491 (.0029) & \tb{.411} (.0038) \\
Last 100 ep. & .481 (.0185)  & .474 (.0111) & \tb{.388} (.0086) \\
\midrule
\multicolumn{4}{l}{CIFAR-100 \citep{cifar100}} \\
\midrule
All episodes    & .872 (.0010) & .856 (.0016) & \tb{.838} (.0021)  \\
Last 100 ep. & .828 (.0024) & .793 (.0031) & \tb{.775} (.0027) \\
\bottomrule
\end{tabular}
\vspace{-0.25cm}
\end{small}
}%
\end{center}
\caption{Avg cost over all episodes and last 100 episodes (lower is better). We report `mean (sem)' over $10$ seeds. }
\vspace{-0.5cm}
\end{wraptable}

We now compare \cb{} with SquareCB \citep{foster2020beyond} and the state-of-the-art CB method FastCB \citep{foster2021efficient}, which respectively minimize the squared loss and log loss for estimating the conditional mean. The key question we investigate here is whether learning the conditional mean via distribution learning with MLE will demonstrate empirical benefit over the non-distributional approaches. We consider three challenging tasks that are all derived from real-world datasets and we briefly describe the construction below.

\paragraph{King County Housing}
This dataset consists of home features and prices, which we normalize to be in $[0,1]$. %
The action space is $100$ evenly spaced prices between $0.01$ and $1.0$.
If the learner overpredicts the true price, the cost is $1.0$.
Else, the cost is $1.0$ minus predicted price.

\paragraph{Prudential Life Insurance}
This dataset contains customer features and an integer risk level in $[8]$, which is our action space.
If the model overpredicts the risk level, the cost is $1.0$.
Otherwise, the cost is $.1\times(y - \hat{y})$ where $y$ is the actual risk level, and $\hat{y}$ is the predicted risk level.

\paragraph{CIFAR-100}
This popular image dataset contains $100$ classes, which correspond to our actions, and each class is in one of $20$ superclasses.
We assign cost as follows: $0.0$ for predicting the correct class, $0.5$ for the wrong class but correct superclass, and $1.0$ for a fully incorrect prediction.

\paragraph{Results}
Across tasks, \cb{} achieves lower average cost over all episodes (\ie, normalized regret) and over the last $100$ episodes (\ie, most updated policies' performance) compared to SquareCB. This indicates the empirical benefit of the distributional approach over the conventional approach based on least square regression, matching the theoretical benefit demonstrated here.
Perhaps surprisingly, \cb{} also consistently outperforms FastCB.
Both methods obtain first-order bounds with the same dependencies on $A$ and $C^\star$, which suggests that \cb{}'s empirical improvement over FastCB cannot be fully explained by existing theory. The only difference between \cb{} and FastCB is that the former integrates online MLE while the latter directly estimates the mean by minimizing the log loss (binary cross-entropy).
An even more fine-grained understanding of the benefits of distribution learning may therefore be helpful in explaining this improvement. \cref{sec:experiment-details} contains all experiment details. Reproducible code is available at \url{https://github.com/kevinzhou497/distcb}.

\section{Conclusion}
We showed that distributional RL leads to small-loss bounds in both online and offline RL, and we also proposed a distributional CB algorithm that outperforms the state-of-the-art FastCB.
A fruitful direction would be to investigate connections of natural policy gradient with our MLE distributional-fitting scheme to inspire a practical offline RL algorithm with small loss guarantees, \emph{\`a la} \citet{cheng2022adversarially}.
Finally, it would be interesting to investigate other loss functions that yield small-loss or even faster bounds.

\paragraph{Acknowledgements}
This material is based upon work supported by the National Science
Foundation under Grant Nos. IIS-1846210 and IIS-2154711.

\bibliographystyle{plainnat}
\bibliography{main}

\newpage
\appendix

\begin{center}\LARGE
\textbf{Appendices}
\end{center}

\section{Notations}

{\renewcommand{\arraystretch}{1.3}%
\begin{table}[h!]
    \centering
      \caption{List of Notations} \vspace{0.3cm}
    \begin{tabular}{l|l}
    $\Scal, \Acal, A$ & State and action spaces, and $A = |\Acal|$. \\
    $\Delta(S)$ & The set of distributions supported by $S$. \\
    $\bar d$ & The expectation of any real-valued distribution $d$, \ie, $\bar d = \Eb[y\sim d]{y}$. \\
    $[N]$ & $\braces{1,2,\dots,N}$ for any natural number $N$. \\
    $Z^\pi_h(x,a)$ & Distribution of $\sum_{t=h}^H c_t$ given $x_h=x,a_h=a$ rolling in from $\pi$. \\
    $Q^\pi_h(x,a),V^\pi_h(x)$ & $Q^\pi_h(x,a)=\bar Z^\pi_h(x,a)$ and $V^\pi_h=\Eb[a\sim\pi(x)]{Q^\pi_h(x,a)}$. \\
    $\pi^\star$ & \multiline{Optimal policy, \ie, $\pi^\star = \argmin_{\pi}V^{\pi}_1(x_1)$. \\
    Without loss of optimality, we take $\pi^\star:\Xcal\to\Acal$ to be Markov \& deterministic.} \\
    $Z^\star_h,Q^\star_h,V^\star_h$ & $Z^\pi_h,Q^\pi_h,V^\pi_h$ with $\pi=\pi^\star$, the optimal policy. \\
    $\Tcal_h^\pi, \Tcal_h^\star$ & The Bellman operators that act on functions. \\
    $\Tcal_h^{\pi,D},\Tcal_h^{\star,D}$ & The distributional Bellman operators that act on conditional distributions. \\
    $V^\pi,Z^\pi,V^\star,Z^\star$ & $V^\pi = V^\pi_1(x_1)$, $Z^\pi=Z^\pi_1(x_1)$. $V^\star,Z^\star$ are defined similarly with $\pi^\star$. \\
    $d^{\pi}_h(x,a)$ & The probability of $\pi$ visiting $(x,a)$ at time $h$. \\
    $C^{\wt\pi}$ & Coverage coefficient $\max_h\nm{\nicefrac{\diff d^{\wt\pi}_h}{\diff\nu_h}}_\infty$. \\
    $D_\triangle(f\Mid g)$ & Triangular discrimination between $f,g$. \\
    $H(f\Mid g)$ & Hellinger distance between $f,g$. \\
    $D_{KL}(f\Mid g)$ & KL divergence between $f,g$.
    \end{tabular}
    \label{tab:notation}
\end{table}
}

\subsection{Statistical Distances}
Let $f,g$ be distributions over $\Ycal$. Then,
\begin{align*}
    &D_\triangle(f\Mid g) = \sum_y \frac{\prns{f(y)-g(y)}^2}{f(y)+g(y)},
    \\&H(f\Mid g) = \sqrt{\frac{1}{2}\sum_y \prns{ \sqrt{f(y)}-\sqrt{g(y)} }^2},
    \\&D_{KL}(f\Mid g) = \sum_y f(y)\log(f(y)/g(y)),
    \\&D_{TV}(f\Mid g) = \frac{1}{2}\sum_y \abs{f(y)-g(y)}.
\end{align*}
The following standard inequalities will be helpful:
\begin{align*}
    &H^2\leq D_{TV}\leq \sqrt{2}H,
    \\&2H^2\leq D_\triangle \leq 4H^2, \tag{\cref{lemma:hellinger-triangle-equiv}}
    \\&H\leq \sqrt{D_{KL}}.
\end{align*}

\begin{lemma}\label{lemma:hellinger-triangle-equiv}
For any distributions $f,g$, we have $2H^2(f\Mid g)\leq D_\triangle(f\Mid g)\leq 4H^2(f\Mid g)$.
\end{lemma}
\begin{proof}
Recall that
\begin{align*}\textstyle
    D_\triangle(f\Mid g)=\int_y \prns{\frac{f(y)-g(y)}{\sqrt{f(y)+g(y)}}}^2.
\end{align*}
Applying $\frac{1}{\sqrt{f(y)}+\sqrt{g(y)}}\leq \frac{1}{\sqrt{f(y)+g(y)}}\leq \frac{\sqrt{2}}{\sqrt{f(y)}+\sqrt{g(y)}}$ concludes the proof.
\end{proof}

\newpage

\section{Modified Algorithms with UAE and for Small Returns Bounds}\label{app:omitted-algs}
In this section, we present the \online{} algorithm with Uniform Action Exploration (UAE).
We also present versions of \online{} and \offline{} for the reward-maximizing setting (instead of the cost-minimizing setting studied throughout the paper); if \textsc{SmallReturn} is turned on, we can derive small-return bounds in \cref{sec:small-reward}.

\begin{algorithm}[!h]
    \caption{\online{} (with UAE and small return)}
    \label{alg:onlinerl_appendix}
    \begin{algorithmic}[1]
        \State\textbf{Input:} number of episodes $K$, distribution function class $\Fcal$, threshold $\beta$, flag \textsc{UAE}, flag \textsc{SmallReturn}.
        \State Initialize $\Dcal_{h,0}\gets\emptyset$ for all $h\in[H]$, and set $\Fcal_0 = \Fcal$.
        \State Set $\op{op}=\max$ if \textsc{SmallReturn} else $\op{op}=\min$.
        \For{episode $k=1,2,\dots,K$}
            \State Set $f^{(k)} = \arg\op{op}_{f\in\Fcal_{k-1}} \op{op}_a \bar f_1(x_1,a)$.
            \State Set $\pi^k_h(x) = \arg\op{op}_a\bar f^{(k)}_{h}(x,a)$.
            \If{\textsc{UAE}}
            \State \multiline{For each $h\in[H]$, collect $x_{h,k}\sim d^{\pi^k}_h,a_{h,k}\sim \op{unif}(\Acal), c_{h,k}\sim C_h(x_{h,k},a_{h,k}),x_{h,k}'\sim P_h(x_{h,k},a_{h,k})$, and augment the dataset $\Dcal_{h,k}=\Dcal_{h,k-1}\cup\braces{(x_{h,k},a_{h,k},c_{h,k},x_{h,k}')}$.}
            \Else
            \State \multiline{Roll out $\pi^k$ and obtain a trajectory $x_{1,k},a_{1,k},c_{1,k},\dots,x_{H,k},a_{H,k},c_{H,k}$. \\ For each $h\in[H]$, augment the dataset $\Dcal_{h,k} = \Dcal_{h,k-1}\cup\braces{(x_{h,k},a_{h,k},c_{h,k},x_{h+1,k})}$. }
            \EndIf
            \State \multiline{For all $(h,f)\in[H]\times\Fcal$, sample $y_{h,i}^f \sim f_{h+1}(x_{h,i}',a')$ and $a'=\arg\op{op}_a \bar f_{h+1}(x_{h,i}',a)$, where $(x_{h,i},a_{h,i},c_{h,i},x_{h,i}')$ is the $i$-th datapoint of $\Dcal_{h,k}$.
            Also, set $z_{h,i}^f=c_{h,i}+y_{h,i}^f$ and define the confidence set,}
            \begin{align*}
                \Fcal_k = \braces{ f\in\Fcal: \sum_{i=1}^k\log f_h(z_{h,i}^f\mid x_{h,i},a_{h,i})\geq \max_{\wt f\in\Fcal}\sum_{i=1}^k\log \wt f_h(z_{h,i}^f\mid x_{h,i},a_{h,i})-7\beta, \forall h\in[H] }.
            \end{align*}
        \EndFor
        \State \textbf{Output:} $\bar\pi = \op{unif}(\pi^{1:K})$.
    \end{algorithmic}
\end{algorithm}

\begin{algorithm}[!h]
\caption{\offline{} (with small return)}
\label{alg:offline_appendix}
\begin{algorithmic}[1]
    \State\textbf{Input:} datasets $\Dcal_1,\dots,\Dcal_H$, distribution function class $\Fcal$, threshold $\beta$, policy class $\Pi$, flag \textsc{SmallReturn}.
    \State \multiline{For all $(h,f,\pi)\in[H]\times\Fcal\times\Pi$, sample $y_{h,i}^{f,\pi} \sim f_{h+1}(x_{h,i}',\pi_{h+1}(x_{h,i}'))$, where $(x_{h,i},a_{h,i},c_{h,i},x_{h,i}')$ is the $i$-th datapoint of $\Dcal_h$.
    Then, set $z_{h,i}^{f,\pi} = c_{h,i} + y_{h,i}^{f,\pi}$ and define the confidence set,}
    \begin{align*}
        &\Fcal_\pi = \braces{ f\in\Fcal: \sum_{i=1}^{N}\log f_{h}(z_{h,i}^{f,\pi}\mid x_{h,i},a_{h,i}) \geq \max_{\wt f\in\Fcal}\sum_{i=1}^{N}\log \wt f_{h}(z_{h,i}^{f,\pi}\mid x_{h,i},a_{h,i})-7\beta, \forall h\in[H] }.
    \end{align*}
    \State Set $\op{op}=\max$ if \textsc{SmallReturn} else $\op{op}=\min$.
    \State For each $\pi\in\Pi$, define the pessimistic estimate $f^\pi = \arg\op{op}_{f\in\Fcal_\pi}\Eb[a\sim\pi(x_1)]{\bar f_1(x_1,a)}$.
    \State \textbf{Output:} $\wh\pi = \arg\op{op}_{\pi\in\Pi}\Eb[a\sim\pi(x_1)]{\bar f^\pi_1(x_1,\pi)}$. \label{line:distbco-policy-selection-appendix}
    \end{algorithmic}
\end{algorithm}

\newpage

\section{Proofs for \cb{}}\label{app:proofs-distcb}
\begin{lemma}[Azuma]\label{lem:mult-azuma}
Let $\braces{X_i}_{i\in[N]}$ be a sequence of random variables supported on $[0,1]$, adapted to filtration $\braces{\Fcal_i}_{i\in[N]}$.
For any $\delta\in(0,1)$, we have w.p. at least $1-\delta$,
\begin{align*}
&\sum_{t=1}^N\Eb{X_t\mid\Fcal_{t-1}}\leq \sum_{t=1}^NX_t+\sqrt{N\log(2/\delta)}, \tag{Standard Azuma}
\\&\sum_{t=1}^N\Eb{X_t\mid\Fcal_{t-1}}\leq 2\sum_{t=1}^NX_t + 2\log(1/\delta). \tag{Multiplicative Azuma}
\end{align*}
\end{lemma}
\begin{proof}
For standard Azuma, see \citet[Theorem 13.4]{tongzhangbook}.
For multiplicative Azuma, apply \citep[Theorem 13.5]{tongzhangbook} with $\lambda = 1$. The claim follows, since $\frac{1}{1-\exp(-\lambda)}\leq 2$.
\end{proof}

\cbRegret*
\begin{proof}[Proof of \cref{thm:distcb}]
First, recall the per-step inequality of ReIGW \citet[Theorem 4]{foster2021efficient}, which states: for any $\wh f$ and $\gamma\geq 2A$, if we set $p = \text{ReIGW}_\gamma(\wh f,\gamma)$, then, for all $f\in[0,1]^A$, we have
\begin{align*}\textstyle
    \sum_a p(a)\prns{f(a)-f(a^\star)} \leq \frac{5A}{\gamma}\sum_ap(a)f(a) + 7\gamma\sum_a p(a)\frac{\prns{\wh f(a)-f(a)}^2}{\wh f(a)+f(a)},
\end{align*}
where $a^\star = \argmin_a f(a)$.
For any $k\in[K]$, applying this to $\wh f = \bar f^{(k)}(s_k,\cdot)$, $p=p_k$ and $f = \bar C(s_k,\cdot)$, we have
\begin{align*}
    \sum_{k=1}^K \Eb[a_k]{\bar C(s_k,a_k)-\bar C(s_k,\pi^\star(s_k))}
    &\leq \sum_{k=1}^K \Eb[a_k]{\frac{5A}{\gamma}\bar C(s_k,a_k) + 7\gamma\frac{\prns{\bar f^{(k)}(s_k,a_k)-\bar C(s_k,a_k)}^2}{\bar f^{(k)}(s_k,a_k)+\bar C(s_k,a_k)} }
    \\&\leq \sum_{k=1}^K \Eb[a_k]{\frac{5A}{\gamma}\bar C(s_k,a_k) + 7\gamma D_\triangle(f^{(k)}(s_k,a_k)\Mid C(s_k,a_k)) } \tag{\cref{eq:tri-disc-ineq-1}}
\end{align*}
Since $D_\triangle\leq 4H^2$, we have
\begin{align*}
    &\sum_{k=1}^K \Eb[a_k]{D_\triangle(f^{(k)}(s_k,a_k)\Mid C(s_k,a_k)) }
    \\&\leq 4\sum_{k=1}^K\Eb[a_k]{ H^2\prns{C(s_k,a_k)\Mid f^{(k)}(s_k,a_k)} }
    \\&\leq 8\sum_{k=1}^K H^2\prns{C(s_k,a_k)\Mid f^{(k)}(s_k,a_k)} + 8\log(1/\delta)  \tag{Multiplicative Azuma, since $H^2\in[0,1]$}
    \\&\leq 8\op{Regret}_{\log}(K) + 10\log(1/\delta). \tag{\citet[Lemma A.14]{foster2021statistical} }
\end{align*}
Hence, we have
\begin{align*}
    \sum_{k=1}^K \Eb[a_k]{\bar C(s_k,a_k)-\bar C(s_k,\pi^\star(s_k))}\leq \frac{5A}{\gamma}\sum_{k=1}^K \Eb[a_k]{ \bar C(s_k,a_k) } + 70\gamma \prns{\op{Regret}_{\log}(K)+\log(1/\delta)}.
\end{align*}
Finally, recalling that $1/(1-\eps)\leq 1+2\eps$ when $\eps\leq \frac12$, and the fact that $\frac{5A}{\gamma}\leq\frac12$, we have
\begin{align*}
    \sum_{k=1}^K \Eb[a_k]{\bar C(s_k,a_k)-\bar C(s_k,\pi^\star(s_k))}\leq \frac{10A}{\gamma}\sum_{k=1}^K\Eb[a_k]{\bar C(s_k,\pi^\star(s_k))} + 140\gamma\prns{\op{Regret}_{\log}(K) + \log(1/\delta)}.
\end{align*}
By Azuma's inequality, we have
\begin{align*}
    &\sum_{k=1}^K \bar C(s_k,a_k)-\bar C(s_k,\pi^\star(s_k))
    \\&\leq 2\sum_{k=1}^K \Eb[a_k]{\bar C(s_k,a_k)-\bar C(s_k,\pi^\star(s_k))} + 2\log(1/\delta)
    \\&\leq \frac{20A}{\gamma}\sum_{k=1}^K\Eb[a_k]{\bar C(s_k,\pi^\star(s_k))} + 140\gamma\prns{\op{Regret}_{\log}(K) + \log(1/\delta)} + 2\log(1/\delta)
    \\&\leq \frac{40A}{\gamma}\prns{C^\star + \log(1/\delta)} + 140\gamma\prns{\op{Regret}_{\log}(K) + \log(1/\delta)} + 2\log(1/\delta). \tag{Multiplicative Azuma}
\end{align*}
Now set $\gamma = \sqrt{\frac{40A\prns{C^\star+\log(1/\delta)}}{140\prns{\op{Regret}_{\log}(K)+\log(1/\delta)}}}\vee 10A$.\\
Case 1 is when $\sqrt{\frac{40A\prns{C^\star+\log(1/\delta)}}{140\prns{\op{Regret}_{\log}(K)+\log(1/\delta)}}}\leq 10A$, \ie, $\prns{C^\star+\log(1/\delta)}\leq 280A\prns{\op{Regret}_{\log}(K)+\log(1/\delta)}$, we have
the above is at most
\begin{align*}
    &4\prns{ C^\star + \log(1/\delta) } + 1120A\prns{\op{Regret}_{\log}(K) + \log(1/\delta)} + 2\log(1/\delta)
    \\&\leq 2240A \prns{\op{Regret}_{\log}(K) + \log(1/\delta)} + 2\log(1/\delta).
\end{align*}
Case 2 is when the left term dominates, then the bound is,
\begin{align*}
    &2\sqrt{4480A\prns{C^\star+\log(1/\delta)}\prns{\op{Regret}_{\log}(K)+\log(1/\delta)}} + 2\log(1/\delta)
    \\&\leq 2\sqrt{13440AC^\star\op{Regret}_{\log}(K)\log(1/\delta) + 4480A \log^2(1/\delta) } + 2\log(1/\delta)
    \\&\leq 232\sqrt{AC^\star\op{Regret}_{\log}(K)\log(1/\delta)} + 134\sqrt{A}\log(1/\delta) + 2\log(1/\delta).
\end{align*}
Putting these two cases together, we have the result.
\end{proof}

\newpage
\section{Distributional Bellman Completeness in low-rank MDPs}\label{sec:dist-bc-low-rank-mdps}

The goal of this section is to show that, under mild conditions in low-rank MDPs, there always exists a function class with bounded bracketing number that satisfies the distributional BC condition.
First, let us recall the low-rank MDP
In this section, we show that linear MDPs automatically satisfy the distributional Bellman completeness assumption.

\lowRankMDPDef*

Suppose that we have a function class $\Phi$ such that $\phi^\star_h\in\Phi$ for all $h$, \ie, $\Phi$ is a realizable function class. For example, in linear MDPs, this is automatically satisfied since we know $\phi^\star$ \emph{a priori}, so $\Phi$ is the singleton with $\phi^\star$. Having a realizable $\Phi$ class is standard for solving low-rank MDPs \citep{uehara2021representation,agarwal2023provable}.

In what follows, let $\alpha = \max_{h,\pi,z,x,a} Z^\pi_h(z\mid x,a)$ denote the maximum density/mass value of the loss-to-go distributions.
Note that $\alpha\geq 1$ always since the mass at $H+1$ is deterministically placed at zero. If we further know that $Z^\pi_h$ is discretely distributed, then $\alpha= 1$. If $Z^\pi_h$ is continuously distributed, we assume it is bounded.

We consider the function class in \cref{eq:linear-mdp-f-class}, which we reproduce here:
\begin{align}
    \mathcal{F}_h^{\text{lin}}=\bigg\{&
	f(z\mid x,a)=\big\langle\phi(x,a),w(z)\big\rangle
	\quad:\quad \phi\in\Phi, w:[0,1]\rightarrow\mathbb{R}^d, \\
	&\text{\; s.t.\; }
	\max_z \|w(z)\|_2\leq \alpha\sqrt{d}
	\text{\; and\; }
\max_{x,a,z}\big\langle\phi(x,a),w(z)\big\rangle\leq \alpha
	\bigg\}. \nonumber
\end{align}
The next lemma (\cref{lem:bell-comp}) shows that this function class satisfies distributional BC.

\begin{lemma}\label{lem:bell-comp}
$\Fcal^{\text{lin}}$ satisfies distributional BC (\cref{ass:policy-dist-bellman-completeness}).
\end{lemma}
\begin{proof}
We denote $\|f\|_\infty = \max_{z,x,a} f(z\given x,a)$. For any $f_{h+1}\in\+F_{h+1}^{\text{lin}}$, we have $\|f_{h+1}\|_\infty\leq\alpha$ by the construction of $\+F_{h+1}^{\text{lin}}$.
Then, let $\Tcal^D$ be either the distributional Bellman operator or distributional optimality operator, the following equalities hold for the appropriate $a'(x')$ based on $\Tcal^D$,

\begin{align*}
	\+T^D f_{h+1}(z\given x,a)
	=&\int_{\Xcal}
	\Pr_h(x'\given x,a) \int_{\RR}\Pr_h(c\mid x,a) f_{h+1}(z-c\given x',a'(x'))\d x'\d c\\
	=&\left\langle\phi^\star_h(x,a),\;\underbrace{ \int_{\Xcal}\mu_h(x')\int_{\RR} \Pr_h(c\mid x,a) f_{h+1}(z-c\given x',a'(x'))\d c \d x'}_{:= w_h(z)} \right\rangle
\end{align*}
Since $\int_{\RR} \Pr_h(c\mid x,a) f_{h+1}(z-c\given x',a'(x'))\d c\leq \|f_{h+1}\|_\infty$, we know that
\begin{align*}
	&\|w_h(z)\|_2\leq\|f_{h+1}\|_\infty \sqrt{d}
	\leq\alpha\sqrt{d}.
\end{align*}
We further note that
\begin{align*}
\max_{x,a,z}\big\langle\phi^\star_h(x,a),w_h(z)\big\rangle=\max_{x,a,z}\+T^D f_{h+1}(z\given x,a)\leq\|f_{h+1}\|_\infty\leq\alpha.
\end{align*}
Also note that $\phi^\star_h\in\Phi$ by realizability. Therefore, $\+T^D f_{h+1}\in\+F_h^{\text{lin}}$, which is the distributional BC condition.
\end{proof}

\subsection{Bounding the bracketing number via discretized rewards}\label{sec:bracketing-number-bound}
We now bound the bracketing number of $\Fcal^{\text{lin}}_h$ under a \emph{discretization assumption that costs and costs-to-gos can only take $M$ many discrete values on an evenly spaced grid}. This can be interpreted as discretizing the reward space, and it can be shown that this discretization error is small for regret or PAC bounds \citep[Section 6]{wang2023near}.
Structural assumptions are necessary to bound the complexity of $\Fcal^{\text{lin}}_h$ and such discretization assumptions are common in practice, \eg, C51 \citep{bellemare2017distributional} and Rainbow \citep{hessel2018rainbow} both set $M=51$ which works well in Atari games.
After discretizing, we can consider $w$ as a mapping from $[M]$, the discrete set on $M$ elements, rather than from the interval $[0,1]$. Note also that since $Z^\pi_h$ are discrete, we have $\alpha=1$.

Now, let $\eps > 0$ be arbitrary and fixed.
Recall that the $\ell_\infty$ bracketing number is equivalent (up to universal constants) to the $\ell_\infty$ covering number, so we will work with the latter.
Let $B(r)$ denote the $d$-dimensional ball of radius $r$ (in $\ell_2$).
Recall that the $\eps$-covering number (in $\ell_2$) of functions $[M]\mapsto B(r)$ scales as $\Ocal((r/\eps)^{dM})$. Let $\Wcal_\eps$ be such the smallest cover.
We can build a $\ell_\infty$ cover of $\Fcal_h^{\text{lin}}$ as follows: $\Ccal_\eps = \braces{ (x,a,z)\mapsto \langle\phi(x,a), w(z)\rangle, w\in\Wcal_\eps, \phi\in\Phi }$.

To check this is a $\eps$ cover, consider any $f\in\Fcal_h^{\text{lin}}$. $f$ corresponds to some $\phi$ and $w$. Let $w'$ be the neighbor of $w$ in $\Wcal_\eps$ and let $f'(x,a,z) = \langle \phi(x,a), w'(z)\rangle$ so indeed $f'\in\Ccal_\eps$. Then, for any $x,a,z$, we have $|\langle \phi(x,a), w(z)-w'(z)\rangle|\leq \|\phi(x,a)\|_2\|w(z)-w'(z)\|_2\leq \eps$. Hence, $\Ccal_\eps$ is an $\ell_\infty$ cover of size $\Ocal((\sqrt{d}/\eps)^{dM}\cdot|\Phi|)$, and so we have shown that $\log N_{[]}(\eps,\Fcal_h^{\text{lin}},\|\cdot\|_\infty) \leq \Ocal(dM\log(d/\eps) + \log|\Phi|)$.

\paragraph{Linear MDPs:} Recall that in linear MDPs, we know the true $\phi^\star$ and so $|\Phi|=1$. Thus, the bracketing number is simply $\Ocal(dM\log(d/\eps))$ in linear MDPs.

\paragraph{Summary and comparison with regular BC: }
In summary, under the assumption that rewards are discretized, we know that low-rank MDPs automatically have distributional function classes that satisfy distributional BC and have bounded bracketing numbers. Furthermore, recall that \citep{wu2023distributional} showed that Linear Quadratic Regulators (LQRs), with deterministic transitions, also have function classes that satisfy distributional BC and have bounded bracketing numbers. Thus, distributional BC holds for the most interesting cases covered by the standard Bellman completeness, \eg, linear MDPs, low-rank MDPs and LQRs. Since learning conditional distributions is statistically harder than learning the conditional mean, we need to pay the price in assuming reward/transitions satisfy regularity assumptions to bound the bracketing number appropriately.

\newpage
\section{Generalization Bounds for Maximum Likelihood Estimation}
This section reviews generalization bounds for the maximum likelihood estimator (MLE).
We adopt the same sequential condition probability estimation setup as in \citet[Appendix E]{agarwal2020flambe}, which we now recall for completeness.
Let $\Xcal$ be the context/feature space and $\Ycal$ be the label space, and we are given a dataset $D = \braces{ (x_i,y_i) }_{i\in[n]}$ from a martingale process:
for $i=1,2,...,n$, sample $x_i \sim \Dcal_i(x_{1:i-1},y_{1:i-1})$ and $y_i\sim p(\cdot \mid x_i)$.
Let $f^\star(x,y) = p(y\mid x)$ and we are given a realizable, \ie, $f^\star\in\Fcal$, function class $\Fcal:\Xcal\times\Ycal\to\Delta(\RR)$ of distributions.
The MLE is an estimate for $f^\star$ that maximizes the log-likelihood objective over our dataset:
\begin{align*}
    \wh f_{\text{MLE}} = \argmax_{f\in\Fcal}\sum_{i=1}^n\log f(x_i,y_i).
\end{align*}

For our guarantees to hold for general hypotheses classes $\Fcal$, we use the bracketing number to quantify the statistical complexity of $\Fcal$ \citep{geer2000empirical}.
\begin{definition}[Bracketing Number]
Let $\Gcal$ be a set of functions mapping $\Xcal\to\RR$.
Given two functions $l,u$ such that $l(x)\leq u(x)$ for all $x\in\Xcal$, the bracket $[l,u]$ is the set of functions $g\in\Gcal$ such that $l(x)\leq g(x)\leq u(x)$ for all $x\in\Xcal$.
We call $[l,u]$ an $\eps$-bracket if $\nm{u-l}\leq\eps$.
Then, the $\eps$-bracketing number of $\Gcal$ with respect to $\nm{\cdot}$, denoted by $N_{[]}(\eps,\Gcal,\nm{\cdot})$ is the minimum number of $\eps$-brackets needed to cover $\Gcal$.
\end{definition}

Since the triangular discrimination is equivalent to squared Hellinger up to universal constants,
we now prove MLE generalization bounds in terms of squared Hellinger.
\begin{lemma}\label{lem:convert-hellinger}
Let $f_1:\Xcal\to \Delta(\Ycal)$ and $f_2:\Xcal\times\Ycal\to\RR_+$ satisfying $\sup_{x\in\Xcal}\int_{\Ycal}f_2(x,y)\diff y\leq s$, then for any distribution $\Dcal\in\Delta(\Xcal)$, we have
\begin{align*}
    \Eb[x\sim\Dcal]{ H^2(f_1(x)\Mid f_2(x,\cdot)) } \leq (s-1)-2\log\EE_{x\sim\Dcal, y\sim f_1(x)}\exp\prns{-\frac{1}{2}\log(f_1(x,y)/f_2(x,y))}.
\end{align*}
\end{lemma}
\begin{proof}
This follows from the proof of \citet[Lemma C.1]{wu2023distributional}.
\end{proof}

\begin{lemma}\label{lem:mle-hellinger}
Fix $\delta\in(0,1)$.
Then w.p. at least $1-\delta$, for any $f\in\Fcal$, we have
\begin{align}
    &\sum_{i=1}^n\Eb[x\sim \Dcal_i]{ H^2(f(x,\cdot)\Mid f^\star(x,\cdot)) } \nonumber
    \\&\quad\leq 6n\epsilon\abs{\Ycal} + 2\sum_{i=1}^n \log\big(f^\star(x_i,y_i)/f(x_i,y_i)\big)+8\log\prns{N_{[]}(\epsilon,\+F,\|\cdot\|_\infty)/\delta}. \label{eq:mle-hellinger-1}
\end{align}
Rearranging, we also have
\begin{align}
    \sum_{i=1}^n \log\big(f(x_i,y_i)/f^\star(x_i,y_i)\big)
    \leq 3n\epsilon\abs{\Ycal}+4\log\prns{N_{[]}(\epsilon,\+F,\|\cdot\|_\infty)/\delta}. \label{eq:mle-hellinger-2}
\end{align}
\end{lemma}
\begin{proof}
We take an $\epsilon$-bracketing of $\+F$, $\{[l_i,u_i]:i=1,2,\dots\}$, and denote $\w~{\+F}=\{u_i:i=1,2,\dots\}$.
Applying Lemma 24 of \cite{agarwal2020flambe} to function class $\w~{\+F}$ and using Chernoff method, w.p. at least $1-\delta$,
for all $\tilde f\in\wt \Fcal$, we have
\begin{equation}\label{eq:lem24}
    \underbrace{-\log\E_{D'}\exp(L(\~f(D),D'))}_{\rm(i)}
    \le
    \underbrace{-L(\~f(D),D)+2\log\prns{N_{[]}(\epsilon,\+F,\|\cdot\|_\infty)/\delta}}_{\rm(ii)}.
\end{equation}
Now, fix any $f\in\Fcal$ and pick $\tilde{f}\in\wt\Fcal$ as the upper bracket, \ie, $f\leq\~f$.
Now set $L(f,D)=\sum_{i=1}^n -\nicefrac{1}{2} \log(f^\star(x_i,y_i)/f(x_i,y_i))$. Then the right hand side of \eqref{eq:lem24} is
 \begin{align*}
     {\rm{(ii)}}=&\frac{1}{2}\sum_{i=1}^n \log(f^\star(x_i,y_i)/\~f(x_i,y_i))+2\log\prns{N_{[]}(\epsilon,\+F,\|\cdot\|_\infty)/\delta}\\
     \le&\frac{1}{2}\sum_{i=1}^n \log(f^\star(x_i,y_i)/f(x_i,y_i))+2\log\prns{N_{[]}(\epsilon,\+F,\|\cdot\|_\infty)/\delta}.
 \end{align*}
 On the other hand, since $H$ is a metric, we have
 \begin{align*}
    &\sum_{i=1}^n\E_{x\sim\+D_i}H^2\left(f(x,\cdot), f^\star(x,\cdot)\right)
    \le\sum_{i=1}^n\E_{x\sim\+D_i}\prns{H\left(f(x,\cdot), \~f(x,y)\right) + H\left(\~f(x,y), f^\star(x,\cdot)\right)}^2\\
    \le&2\underbrace{\sum_{i=1}^n\E_{x\sim\+D_i}H^2\left(f(x,\cdot), \~f(x,y)\right)}_{\rm{(iii)}}
    +
    2\underbrace{\sum_{i=1}^n\E_{x\sim\+D_i}H^2\left(\~f(x,y), f^\star(x,\cdot)\right)}_{\rm{(iv)}}.
 \end{align*}

 For $\rm{(iii)}$, by the definition, we have $\~f(x,y)-f(x,y)\in[0,\epsilon]$ for all $x$, so
 \begin{align*}
     {\rm(iii)}=\sum_{i=1}^n\E_{x\sim\+D_i}H^2\left(f(x,\cdot), \~f(x,y)\right)
     \leq \sum_{i=1}^n\E_{x\sim\+D_i}2\int_y\abs{f(x,y)-\~f(x,y)}\diff y
     \leq 2n\epsilon\abs{\Ycal}.
 \end{align*}

 For $\rm{(iv)}$, we apply \cref{lem:convert-hellinger} with $f_1=f^\star$ and $f_2=\~f$ (thus $s=1+\epsilon|\+Y|$) and get
 \begin{align*}
    {\rm(iv)}
    =&n\epsilon|\+Y|-2\sum_{i=1}^n\log\E_{x,y\sim f^\star(x,\cdot)}\exp\left(-\frac{1}{2}\log\left(f^\star(x,y)/\~f(x,y)\right)\right)\\
    =&n\epsilon|\+Y|-2\sum_{i=1}^n\log\E_{x,y\sim\+D_i}\exp\left(-\frac{1}{2}\log\left(f^\star(x,y)/\~f(x,y)\right)\right)\\
    =&n\epsilon|\+Y|-2\log\E_{x,y\sim\+D'}\left[\exp\left(\sum_{i=1}^n-\frac{1}{2}\log\left(f^\star(x,y)/\~f(x,y)\right)\right)\middle|D\right]\\
    =&n\epsilon|\+Y|+2\cdot{\rm{(i)}}.
\end{align*}
By plugging $\rm{(iii)}$ and $\rm{(iv)}$ back we get
\begin{align*}
    \sum_{i=1}^n\E_{x\sim\+D_i}H^2\left(f(x,\cdot), f^\star(x,\cdot)\right)
    \le 6n\epsilon|\+Y|+4\cdot{\rm{(i)}}.
\end{align*}
Notice that $\rm{(i)}\le\rm{(ii)}$, so we complete the proof by plugging $\rm(ii)$ into the above.
\end{proof}

We first state the MLE generalization result for finite $\Fcal$.
\begin{theorem}\label{thm:mle-version-space-finite}
Suppose $\Fcal$ is finite.
Fix any $\delta\in(0,1)$, set $\beta=\log(|\+F|/\delta)$ and define
\begin{align*}
    \wh\Fcal = \braces{ f\in\Fcal: \sum_{i=1}^n\log f(x_i,y_i)\geq \max_{\wt f\in\Fcal}\sum_{i=1}^n \wt f(x_i,y_i) - 4\beta }.
\end{align*}
Then w.p. at least $1-\delta$, the following holds:
\begin{enumerate}
    \item[(1)] The true distribution is in the version space, i.e., $f^\star\in\widehat{\mathcal{F}}$.
    \item[(2)] Any function in the version space is close to the ground truth data-generating distribution, \ie, for all $f\in\wh\Fcal$
    \begin{align*}
        \sum_{i=1}^n\Eb[x\sim\Dcal_i]{ H^2( f(x,\cdot) \Mid f^\star(x,\cdot) ) }\leq 22\beta.
    \end{align*}
\end{enumerate}
\end{theorem}
\begin{proof}
These two claims follow from \cref{lem:mle-hellinger} with $\epsilon=0$, and so $N_{[]}(\epsilon,\Fcal,\|\cdot\|_\infty)=|\Fcal|$.
For (1), apply \cref{eq:mle-hellinger-2} to $f=\wh f_{\text{MLE}}$ to see that $f^\star\in\wh\Fcal$.
For (2), apply \cref{eq:mle-hellinger-1} and note that the sum term is at most $4\beta$.
Thus, the right hand side of \cref{eq:mle-hellinger-1} is at most $(6+8+8)\beta=22\beta$.
\end{proof}

We now state the result for infinite $\Fcal$ using bracketing entropy.
\begin{theorem}\label{thm:mle-version-space-general}
Fix any $\delta\in(0,1)$, set $\beta=\log(N_{[]}((n|\Ycal|)^{-1},\Fcal,\|\cdot\|_\infty)/\delta)$ and define
\begin{align*}
    \wh\Fcal = \braces{ f\in\Fcal: \sum_{i=1}^n\log f(x_i,y_i)\geq \max_{\wt f\in\Fcal}\sum_{i=1}^n \wt f(x_i,y_i) - 7\beta }.
\end{align*}
Then w.p. at least $1-\delta$, the following holds:
\begin{enumerate}
    \item[(1)] The true distribution is in the version space, i.e., $f^\star\in\widehat{\mathcal{F}}$.
    \item[(2)] Any function in the version space is close to the ground truth data-generating distribution, \ie, for all $f\in\wh\Fcal$
    \begin{align*}
        \sum_{i=1}^n\Eb[x\sim\Dcal_i]{ H^2( f(x,\cdot) \Mid f^\star(x,\cdot) ) }\leq 28\beta.
    \end{align*}
\end{enumerate}
\end{theorem}
\begin{proof}
These two claims follow from \cref{lem:mle-hellinger} with $\epsilon=\nicefrac{1}{n|\Ycal|}$.
For (1), apply \cref{eq:mle-hellinger-2} to $f=\wh f_{\text{MLE}}$ to see that $f^\star\in\wh\Fcal$.
For (2), apply \cref{eq:mle-hellinger-1} and note that the sum term is at most $7\beta$.
Thus, the right hand side of \cref{eq:mle-hellinger-2} is at most $(6+14+8)\beta=28\beta$.
\end{proof}

\section{Confidence set construction with general function class}\label{sec:confidence-set-infinite-functions}
In this section, we extend the confidence set construction of \online{} and \offline{} to general $\Fcal$, which can be infinite. Our procedure constructs the confidence set by performing the thresholding scheme on an $\eps$-net of $\Fcal$. While constructing an $\eps$-net for $\Fcal$ is admittedly a computationally hard procedure, this is still information theoretically possible and our focus in \online{} and \offline{} is to show that distributional RL information-theoretically leads to small-loss bounds.

We first define some notations.
Let $\Fcal^\downarrow$ and $\Fcal^\uparrow$ denote a lower and upper $\eps$-bracketing of $\Fcal$, \ie, for any $f\in\Fcal$, there exists an $\eps$-bracket $[f^\downarrow, f^\uparrow]$
such that for all $h$, $f^\downarrow_h\leq f_h\leq f^\uparrow_h$ with $f^\downarrow\in\Fcal^\downarrow,f^\uparrow\in\Fcal^\uparrow$.
Recall that a lower bracket $g\in\Fcal^\downarrow$ may not be a valid distribution, but since elements of $\Fcal$ map to non-negative values, we can assume $g$ has non-negative entires as well.
Also, we have $\alpha^g_h(x,a):=\int g_h(z\mid x,a)\geq 1-\eps$, so for $\eps$ small enough, $g$ is normalizable.
Hence, define $\wt g(z\mid x,a) = \alpha^g_h(x,a)^{-1}g(z\mid x,a)$ as the normalized version, which is a valid distribution that we can sample from.

Now, consider any martingale $\braces{x_{h,i},a_{h,i},c_{h,i}}_{i\in[n], h\in[H]}$, which could be the online data up to episode $k$ or the offline data (consisting of $N$ i.i.d. samples).
We define the MLE with respect to a lower bracket element as follows.
For any $h\in[H],g\in\Fcal^\downarrow,\pi\in\Pi$, sample $y_{h,i}^{g,\pi}\sim \wt g_{h+1}(x_{h,i}',\pi(x_{h,i}'))$,
and $z_{h,i}^{g,\pi}=c_{h,i}+y_{h,i}^{g,\pi}$, define the MLE solution for $(g,\pi)$ at time $h$ as,
\begin{align*}
    \textsc{Mle}_h^{g,\pi} = \argmax_{f\in\Fcal} \sum_{i=1}^n \log f_h(z_{h,i}^{g,\pi}\mid x_{h,i},a_{h,i}).
\end{align*}
Also, define the version space with respect to the above MLE as,
\begin{align*}
    \Fcal_{g,\pi,h} = \braces{ f\in\Fcal:\sum_{i=1}^n\log f_h(z_{h,i}^{g,\pi}\mid x_{h,i},a_{h,i})\geq \sum_{i=1}^n\log \textsc{Mle}_h^{g,\pi}(z_{h,i}^{g,\pi}\mid x_{h,i},a_{h,i})-\beta }.
\end{align*}

We now prove a key result that implies that $\Tcal^{\pi}_hf^\downarrow_{h+1}$ falls into the confidence set $\Fcal_{f^\downarrow,\pi,h}$.
\begin{theorem}\label{thm:bracket-mle}
For any $\delta\in(0,1)$ and suppose $n\geq 2$.
Then, w.p. at least $1-\delta$, for any $h\in[H],g\in\Fcal,f^\downarrow\in\Fcal^\downarrow,\pi\in\Pi$, we have
\begin{align*}
    \sum_{i=1}^n\log g_h(z_{h,i}^{f^\downarrow,\pi}\mid x_{h,i},a_{h,i}) - \log\Tcal_h^\pi f^\downarrow_{h+1}(z_{h,i}^{f^\downarrow,\pi}\mid x_{h,i},a_{h,i})\leq \log(e^4N_{[]}(n^{-1},\Fcal,\nm{\cdot}_\infty)^2\abs{\Pi}/\delta).
\end{align*}
where $z_{h,i}^{f^\downarrow,\pi}=c_{h,i}+y_{h,i}^{f^\downarrow,\pi}$ and $y_{h,i}^{f^\downarrow,\pi}\sim \wt f^\downarrow_{h+1}(\cdot\mid x_{h,i}',\pi_{h+1}(x_{h,i}'))$.
\end{theorem}
\begin{proof}[Proof of \cref{thm:bracket-mle}]
Consider a $\eps$-bracketing of $\Fcal$ where $\eps\leq 1/n\leq 1/2$; we will study each element and conclude with a union bound.
For any lower bracket $l$ and upper bracket $u$ in the bracketing (note $l,u$ need not correspond to the same bracket).
Recall that $\alpha^l_{h+1}(x,a):=\int l_{h+1}(z\mid x,a)$, so we have $1-\eps\leq \alpha^l_{h+1}\leq 1$ since $l$ is a lower $\eps$-bracket of distributions.
Therefore, we have
\begin{align*}
    \Eb{ \exp\sum_{i=1}^n\log\prns{ \frac{u_h(z_{h,i}^{l,\pi}\mid x_{h,i},a_{h,i})}{\Tcal^\pi_h l_{h+1}(z_{h,i}^{l,\pi}\mid x_{h,i},a_{h,i})} } }
    =\prod_{i=1}^n\Eb[\nu_{h,i}]{ \frac{u_h(z_{h,i}^{l,\pi}\mid x_{h,i},a_{h,i})}{\Tcal^\pi_h l_{h+1}(z_{h,i}^{l,\pi}\mid x_{h,i},a_{h,i})} },
\end{align*}
where $\nu_{h,i}$ is the distribution of data from $i$-th round and time $h$.
Note that $\nu_{h,i}(x,a,c,x')=d_{h,i}(x,a)C_h(c\mid x,a)P_h(x'\mid x,a)$ for some distribution $d_{h,i}(x,a)$.
Now focus on each $i$, so for all $i$, we have
\begin{align*}
    &\Eb[\nu_{h,i}]{ \frac{u_h(z_{h,i}^{l,\pi}\mid x_{h,i},a_{h,i})}{\Tcal^\pi_h l_{h+1}(z_{h,i}^{l,\pi}\mid x_{h,i},a_{h,i})} }
    \\&= \int_{x,a,c,x',y} \nu_{h,i}(x,a,c,x')\wt l_{h+1}(y\mid x',\pi(x')) \frac{u_h(c+y\mid x,a)}{\int_{c,x'} \nu_{h,i}(c,x'\mid x,a)l_{h+1}(y\mid x',\pi(x'))}
    \\&= \int_{x,a,z} d_{h,i}(x,a)\int_z u_h(z\mid x,a)
    \\&\times \int_{c,x'}\nu_{h,i}(c,x'\mid x,a)\wt l_{h+1}(z-c\mid x',\pi(x')) \frac{1}{\int_{c,x'} \nu_{h,i}(c,x'\mid x,a)l_{h+1}(z-c\mid x',\pi(x'))}
    \\&= \int_{x,a,z} d_{h,i}(x,a)\int_z u_h(z\mid x,a) \alpha^l_{h+1}(x,a)^{-1}
    \\&\leq \frac{1+\eps}{1-\eps} = 1+\frac{2\eps}{1-\eps}\leq 1+\frac{4}{n}.
\end{align*}
Therefore,
\begin{align*}
    \Eb{ \exp\sum_{i=1}^n\log\prns{ \frac{u_h(z_{h,i}^{l,\pi}\mid x_{h,i},a_{h,i})}{\Tcal^\pi_h l_{h+1}(z_{h,i}^{l,\pi}\mid x_{h,i},a_{h,i})} } }
    &\leq \prns{1+4/n}^n\leq e^4.
\end{align*}
Thus, by Markov's inequality, w.p. at least $1-\delta$, we have
\begin{align*}
    \sum_{i=1}^n\log\prns{ \frac{u_h(z_{h,i}^{l,\pi}\mid x_{h,i},a_{h,i})}{\Tcal^\pi_h l_{h+1}(z_{h,i}^{l,\pi}\mid x_{h,i},a_{h,i})} }\leq \ln(e^4/\delta).
\end{align*}
To conclude, apply union bound to get this result for all brackets.
\end{proof}

For the remainder of this section, we assume the policy class $\Pi$ is finite.
However, it is possible to extend our results using policy covers in the Hamming distance; in that case, $\log\abs{\Pi}$ would be replaced by the log covering number or entropy integral of $\Pi$ \citep[as in][]{zhou2023offline,kallus2022doubly}.
We note that for the \emph{online} case, we rely on the assumption that for any $f\in\Fcal$ we have $\pi^f\in\Pi$, where recall that $\pi_h^f(x) = \argmin_a \bar f_h(x,a)$.
This is because $\Tcal^{\star,D}$ is not a contraction so we cannot operate with $\Tcal^{\star,D}$ directly and instead operate with $\Tcal^{\pi^f,D}$.
We highlight that this assumption is automatically satisfied in tabular MDPs, since the whole policy space is finite, and $\log\abs{\Pi}=\Ocal(X\log(A))$ is lower order compared to log of the bracketing entropy of $\Fcal_{tab}$, which is $\Ocal(X^2A^2)$.
In contrast, in non-distributional methods such as GOLF, the regular Bellman optimality operator is a contraction so standard Lipschitz arguments for covering go through.
We note that it is also possible to construct covers of $\Fcal$ in the Hellinger distance, but the metric entropy of $\Fcal_{tab}$ seems to be on the same order as its bracketing entropy.

We now describe the version space construction for general $\Fcal$, first for the online setting.
Fix any $k$, and define the set
\begin{align*}
    \Fcal_{f^\downarrow,\pi,h} = \braces{f\in\Fcal: \sum_{i=1}^k\log f_h(z_{h,i}^{f^\downarrow,\pi}\mid x_{h,i},a_{h,i}) \geq
    \sum_{i=1}^k\log \textsc{Mle}^{f^\downarrow,\pi}_h(z_{h,i}^{f^\downarrow,\pi}\mid x_{h,i},a_{h,i})-\beta }
\end{align*}
Then, construct the version space as
\begin{align*}
    \Fcal_k = \braces{ f\in\Fcal: f_h\in\Fcal_{f^\downarrow,\pi^f,h}, \forall h\in[H] }.
\end{align*}

\begin{theorem}\label{thm:distgolf-mle}
Fix any $\delta\in(0,1)$ and suppose \cref{ass:policy-dist-bellman-completeness}.
Set $\beta=\log(KH\cdot N_{[]}(K^{-1},\Fcal,\|\cdot\|_\infty)\abs{\Pi}/\delta)$.
Then, w.p. at least $1-\delta$, the following holds:
\begin{enumerate}
    \item[(1)] The optimal cost distribution is in the version space, i.e., $Z^\star\in\Fcal_k$.
    \item[(2)] For all $f\in\Fcal_k$ and $h\in[H]$,
    \begin{align*}
        \sum_{i=1}^{k}\Eb[\pi^i]{ H^2( f_h(x_h,a_h) \Mid \Tcal^{\star,D}_hf_{h+1}(x_h,a_h) ) }\leq 60\beta.
    \end{align*}
\end{enumerate}
\end{theorem}
\begin{proof}
First, we want to verify that $Z^\star\in\Fcal_k$.
Let $f^\downarrow$ be the lower bracket of $Z^\star$ and set $g = \textsc{Mle}^{f^\downarrow,\pi^\star}_h\in\Fcal$; note $\pi^\star=\pi^{Z^\star}$.
By \cref{thm:bracket-mle}, we have $\sum_{i=1}^k \log\textsc{Mle}^{f^\downarrow,\pi^\star}_h(z^{f^\downarrow,\pi^\star}_{h,i}\mid x_{h,i},a_{h,i})-\log \Tcal^{\pi^\star,D}_h f^\downarrow_{h+1}(z^{f^\downarrow,\pi^\star}_{h,i}\mid x_{h,i},a_{h,i})\leq\Ocal(\beta)$.
Therefore, noting that $Z^\star_h = \Tcal^{\pi^\star,D}_h Z^\star_{h+1}\geq \Tcal^{\pi^\star,D}_h f^\downarrow_{h+1}$ shows that $Z^\star_h\in\Fcal_{f^\downarrow,\pi^\star,h}$ for every $h$, implying that $Z^\star\in\Fcal_k$.

For the second claim, fix any $f\in\Fcal_k$ and $h\in[H]$.
Then,
\begin{align*}
    &\sum_{i=1}^k\Eb[\pi^i]{H^2(f_h(x_h,a_h)\Mid \Tcal^{\star,D}_h f_{h+1}(x_h,a_h))}
    \\&= \sum_{i=1}^k\Eb[\pi^i]{H^2(f_h(x_h,a_h)\Mid \Tcal^{\pi^f,D}_h f_{h+1}(x_h,a_h))}
    \\&\leq 2\sum_{i=1}^k\Eb[\pi^i]{H^2(f_h(x_h,a_h)\Mid \Tcal^{\pi^f,D}_h \wt f_{h+1}^\downarrow(x_h,a_h)) + H^2(\Tcal^{\pi^f,D}_h \wt f_{h+1}^\downarrow(x_h,a_h)\Mid \Tcal^{\pi^f,D}_h f_{h+1}(x_h,a_h))}
    \\&\leq 2(28\beta + 3k\eps).
\end{align*}
The $\beta$ comes from \cref{thm:mle-version-space-general}, and for $\eps$, we used the fact that $H^2\leq H\leq TV$, and
\begin{align*}
    &\sum_{i=1}^k\Eb[\pi^i]{TV(\Tcal^{\pi^f,D}_h \wt f_{h+1}^\downarrow(x_h,a_h)\Mid \Tcal^{\pi^f,D}_h f_{h+1}(x_h,a_h))}
    \\&=\sum_{i=1}^k\EE_{\pi^i}\int_z \abs{\Tcal^{\pi^f,D}_h \wt f_{h+1}^\downarrow(z\mid x_h,a_h)-\Tcal^{\pi^f,D}_h f_{h+1}(z\mid x_h,a_h))}
    \\&= \sum_{i=1}^k\EE_{\pi^i}\int_z \sum_{c,x'} \nu(c,x'\mid x_h,a_h)\abs{\wt f_{h+1}^\downarrow(z-c\mid x',\pi^f(x'))-f_{h+1}(z-c\mid x',\pi^f(x'))}
    \\&\leq \sum_{i=1}^k 3\eps = 3k\eps,
\end{align*}
since for any $x,a$, we have $\int_z \abs{\wt f_{h+1}^\downarrow(z\mid x,a)-f_{h+1}(z\mid x,a)}\leq 3\eps$.
There are two cases. If $\wt f_{h+1}^\downarrow(z\mid x,a)\geq f_{h+1}(z\mid x,a)$, then $\wt f_{h+1}^\downarrow(z\mid x,a)-f_{h+1}(z\mid x,a)\leq (1-\eps)^{-1}f^\downarrow_{h+1}(z\mid x,a)-f_{h+1}(z\mid x,a)\leq 2\eps f_{h+1}(z\mid x,a)$ since $(1-\eps)^{-1}\leq 1+2\eps$.
If $\wt f_{h+1}^\downarrow(z\mid x,a)< f_{h+1}(z\mid x,a)$, then $f_{h+1}(z\mid x,a)-\wt f_{h+1}^\downarrow(z\mid x,a)\leq f_{h+1}(z\mid x,a)-f_{h+1}^\downarrow(z\mid x,a)\leq \eps$. Thus, $\int_z \max(2\eps f_{h+1}(z\mid x,a), \eps)\leq \int_z 2\eps f_{h+1}(z\mid x,a) + \eps= 3\eps$.
Thus, setting $\eps=1/K$ gives
\begin{align*}
    \sum_{i=1}^k\Eb[\pi^i]{H^2(f_h(x_h,a_h)\Mid \Tcal^{\star,D}_h f_{h+1}(x_h,a_h))}\leq 59\beta.
\end{align*}
\end{proof}

For the offline setting, fix any $\pi$ and define its general version space as,
\begin{align*}
    \Fcal_\pi = \braces{ f\in\Fcal: f_h\in\Fcal_{f^\downarrow,\pi,h}, \forall h\in[H] }.
\end{align*}
\begin{theorem}\label{thm:distbco-mle}
Fix any $\delta\in(0,1)$ and suppose \cref{ass:policy-dist-bellman-completeness}.
Set $\beta=\log(H|\Pi|\cdot N_{[]}((n|\Ycal|)^{-1},\Fcal,\|\cdot\|_\infty)/\delta)$.
Then, w.p. at least $1-\delta$, the following holds for all policies $\pi\in\Pi$:
\begin{enumerate}
    \item[(1)] The policy cost distribution is in the version space, i.e., $Z^\pi\in\Fcal_\pi$.
    \item[(2)] Any function in the version space has bounded triangular discrimination with the ground truth data-generating distribution, \ie, for all $f\in\Fcal_\pi$ and $h\in[H]$,
    \begin{align*}
        \Eb[\nu_h]{ H^2( f_h(x_h,a_h) \Mid \Tcal^{\pi,D}_hf_{h+1}(x_h,a_h) ) }\leq 60\beta N^{-1}.
    \end{align*}
\end{enumerate}
\end{theorem}
\begin{proof}
The proof is the same as in \cref{thm:distgolf-mle}, but instead of $\pi^f$, we fix any $\pi$.
\end{proof}

\section{The $\ell_p$ distributional eluder dimension}

Let $\Scal$ denote any input space (for example, we will later instantiate $\Scal=\Xcal$ or $\Scal=\Xcal\times\Acal$).
Let $\Psi$ denote a set of functions mapping from $\Scal\to\RR$. Let $\Dcal$ be a set of distributions on $\Scal$.

Recall the definition of $\eps$-independent sequence (of distributions) from \citet{jin2021bellman}.
\begin{definition}[$\ell_2$-independent sequence]
A distribution $\nu\in\Dcal$ is $(\eps,\ell_2)$-independent of a sequence $\braces{d^{(1)},\dots,d^{(n)}}\subset\Dcal$ if there exists $\psi\in\Psi$ such that $\abs{ \EE_\nu\psi }>\eps$ and also $\sqrt{\sum_{i=1}^n\prns{\EE_{d^{(i)}}\psi}^2}\leq\eps$.
\end{definition}
Note that the definition is on sequences of distributions, which generalizes the original definition on sequences of points from \citet{russo2013eluder}.

We now generalize the above definition for the general $\ell_p$ norm.
\begin{definition}[$\ell_p$-independent sequence]
A distribution $\nu\in\Dcal$ is $(\eps,\ell_p)$-independent of a sequence $\braces{d^{(1)},\dots,d^{(n)}}\subset\Dcal$ if there exists $\psi\in\Psi$ such that $\abs{ \EE_\nu\psi }>\eps$ and also $\sum_{i=1}^n\abs{\EE_{d^{(i)}}\psi}^p\leq\eps^p$.
\end{definition}
Using the definition of independent sequences established so far, we define the $\ell_p$ distributional eluder dimension.
\begin{definition}[$\ell_p$-distributional eluder dimension]
For any $p$, define the $\ell_p$-distributional eluder dimension (denoted by $\DE_p(\Psi,\Dcal,\eps)$) as the length of the longest sequence $\{d^{(1)},\dots,d^{(d)}\}\subset\Dcal$ such that there exists $\eps'\geq\eps$, such that for all $t\in[d]$, $d^{(t)}$ is $(\eps',\ell_p)$-independent of $d^{(1)},\dots,d^{(t-1)}$.
\end{definition}
Of particular interest to us is the $\ell_1$ case.
We show that the $\ell_1$ eluder dimension is dominated by the $\ell_2$ eluder dimension of \citet{jin2021bellman}.
\eluderOneGeneralizesTwo*
\begin{proof}
Since $\sqrt{\sum_i x_i^2}\leq\sum_i\abs{x_i}$, we have that any witness (long independent sequence) for $\ell_1$ is also a witness for $\ell_2$. So, the maximum length of the $\ell_2$ witnesses is longer than the $\ell_1$ witnesses. %
\citet[Proposition 19]{liu2022partially} obtains an analogous result for the non-distributional eluder dimension of \citet{russo2013eluder}.
\end{proof}

We now prove the key pigeonhole result for the $\ell_1$ distributional eluder dimension.
\eluderPigeonhole*
\begin{proof}
For any $\Gamma\subset\Dcal$, $\nu\in\Dcal$, and $0<\eps\leq 1$, let $L(\nu,\Gamma,\eps)$ denote the number of disjoint subsets of $\Gamma$ such that each subset is $\eps$-dependent of $\nu$, \ie, for all such disjoint subsets of $\Gamma$, it is not the case that $\nu$ is $(\eps,\ell_1)$-independent of each subset.

\textbf{Fact 1: For any $\eps$, if $\abs{\EE_{d^{(k)}}f^{(k)}}>\eps$ for some $k\in[K]$, then $L(d^{(k)},d^{(1:k-1)},\eps)<\beta/\eps$.}
\newline By definition of $L:=L(d^{(k)},d^{(1:k-1)},\eps)$, there exist disjoint subsequences $\mathfrak{G}^{(1)},\dots,\mathfrak{G}^{(L)}$ of $d^{(1:k-1)}$ such that each subsequence $\mathfrak{G}^{(i)}$ satisfies $\sum_{d\in\mathfrak{G}^{(i)}}\abs{\EE_d f^{(k)}}>\eps$. Therefore, summing over all subsequences, we have $L\eps < \sum_{i=1}^{k-1}\abs{\EE_{d^{(i)}}f^{(k)}}\leq\beta$, where the $\beta$ inequality comes from the premise. This proves Fact 1.

\textbf{Fact 2: For any $\eps$ and any sequence $\braces{\nu^{(1)},\dots,\nu^{(\kappa)}}\subset\Dcal$, there exists $j\in[\kappa]$ such that $L(\nu^{(j)}, \nu^{(1:j-1)},\eps) \geq J := \floor*{ (\kappa-1)/\DE_1(\Psi,\Dcal,\eps) }$. }
\newline If $J=0$, the claim is vacuously true.
Otherwise, consider the following algorithm for finding the $j$:
\begin{enumerate}[label=Step \arabic*)]
    \item Initialize $\mathfrak{G}^{(1)}=[\nu^{(1)}],\dots,\mathfrak{G}^{(J)}=[\nu^{(J)}]$ and let $j=J+1$.
    \item If $\nu^{(j)}$ is $\eps$-dependent on all of $\mathfrak{G}^{(i)},i\in[J]$, then the claim is proven and terminate.
    \item Otherwise, there exists some $\mathfrak{G}^{(i)},i\in[J]$ such that $\nu^{(j)}$ is $\eps$-independent of it. Append $\nu^{(j)}$ to $\mathfrak{G}^{(i)}$, \ie, $\mathfrak{G}^{(i)} = \mathfrak{G}^{(i)} + [\nu^{(j)}]$. Increment $j=j+1$ and go back to Step 2.
\end{enumerate}
Hence, we need to argue this process terminates at Step 2 before $j$ gets to $\kappa+1$.
We prove this by contradiction: assume $j$ gets to $\kappa+1$. Let $i\in[J]$ be such that $\mathfrak{G}^{(i)}$ has the most elements (break ties arbitrarily).
Since $\kappa=\sum_{i=1}^J\abs{\mathfrak{G}^{(i)}}\leq J\abs{\mathfrak{G}^{(i)}}$, we have that $\abs{\mathfrak{G}^{(i)}}\geq\kappa/J\geq \frac{\kappa}{\kappa-1} \DE_1(\Psi,\Dcal,\eps)>\DE_1(\Psi,\Dcal,\eps)$, where we've also used the definition of $J$.
By construction, $\mathfrak{G}^{(i)}$ is an $\eps$-eluder sequence, \ie, it is a sequence such that each element is $\eps$-independent of its predecessors. However, this is a contradiction because its size is greater than $\DE_1(\Psi,\Dcal,\eps)$. Therefore, this process terminates at Step 2 for some $j$, which is the witness for proving Fact 2.

\textbf{Fact 3: For any $\eps$ and $k\in[K]$, we have $\sum_{t=1}^k\II\bracks{ \abs{\EE_{d^{(t)}}f^{(t)}} > \eps }\leq \prns{\beta\eps^{-1} + 1}\DE_1(\Psi,\Dcal,\eps)+1$.}
\newline Fix any $\eps$ and $k\in[K]$.
Let $\braces{d^{(i_1)}, \dots, d^{(i_\kappa)}}$ be all the elements of $d^{(1:k)}$ such that $\EE_{d^{(t)}}f^{(t)}>\eps$ for $t=i_1,\dots,i_\kappa$.
By Fact 2, there exists $j\in[\kappa]$ such that $L(d^{(i_j)},d^{(i_{1:j-1})},\eps)\geq \floor*{ (\kappa-1)/\DE_1(\Psi,\Dcal,\eps) }$.
By Fact 1, we have $L(d^{(i_j)}, d^{(1:i_j)}, \eps) \leq \beta/\eps$.
Finally notice that $L(d^{(i_j)},d^{(i_{1:j-1})},\eps)\leq L(d^{(i_j)}, d^{(1:i_j)}, \eps)$ since adding more elements can only create more $\eps$-dependent-of-$\nu$ disjoint subsets. Thus, combining these inequalities, we have $\floor*{ (\kappa-1)/\DE_1(\Psi,\Dcal,\eps) } < \beta/\eps$. This implies $\kappa\leq (\beta\eps^{-1}+1)\DE_1(\Psi,\Dcal,\eps)+1$, which proves Fact 3.

\textbf{Finishing the proof}
\newline Fix any $k\in[K]$ and $\omega>0$. We have
\begin{align*}
    \sum_{t=1}^k\abs{\EE_{d^{(t)}}f^{(t)}}
    &= \sum_{t=1}^k\int_0^C\II\bracks{ \abs{\EE_{d^{(t)}}f^{(t)}} > y }\diff y
    \\&\leq k\omega + \sum_{t=1}^k\int_\omega^C\II\bracks{ \abs{\EE_{d^{(t)}}f^{(t)}} > y }\diff y
    \\&= k\omega + \int_\omega^C \sum_{t=1}^k\II\bracks{ \abs{\EE_{d^{(t)}}f^{(t)}} > y }\diff y
    \\&\leq k\omega + \int_\omega^C \braces{\prns{\beta/y + 1}\DE_1(\Psi,\Dcal,y) + 1} \diff y \tag{Fact 3}
    \\&\leq k\omega + \int_\omega^C \braces{\prns{\beta/y + 1}\DE_1(\Psi,\Dcal,\omega) + 1} \diff y \tag{Monotonicity of $\DE_1$}
    \\&\leq k\omega + (d+1)C + d\beta\log(C/\omega). \tag{$d:=\DE_1(\Psi,\Dcal,\omega)$}
\end{align*}
This completes the proof.
\end{proof}

\subsection{Bounding V-type $\ell_2$ eluder dimension in low-rank MDPs}

\begin{theorem}[Bound of $\ell_2$ distributional eluder for low-rank MDPs]\label{thm:low-rank-mdp-eluder}
Suppose the MDP is a low-rank MDP.
Let $\Psi\subset \Xcal\to[0,1]$ be any class of functions mapping $\Xcal$ to $[0,1]$.
Suppose $\Dcal=\braces{x\mapsto d^\pi_h(x): \pi\in\Pi}$ for some $h\in[H]$.
Then, we have
\begin{equation}
    \DE_2(\Psi,\Dcal,\eps)\leq\Ocal(d\log(d/\eps)).
\end{equation}
\end{theorem}
\begin{proof}
If $h=1$, then $\Dcal$ is a singleton. Hence, $\DE_2(\Psi,\Dcal,\eps)\leq 1$.
Hence, suppose $h\geq 2$; set $h:=h-1$ and we will focus on $d^\pi_{h+1}$ in the remainder.
Suppose $\braces{d^{(k)},f^{(k)}}_{k\in[T]}$ is any sequence such for all $k\in[T]$, we have that $(d^{(k)},f^{(k)})$ is $(\eps,\ell_2)$-independent of its predecessors.
For any $k$, set $\Sigma_k = \sum_{i=1}^{k-1}\Eb[d^{(i)}]{ \phi^\star_h(x_h,a_h)}\Eb[d^{(i)}]{\phi^\star_h(x_h,a_h) }^\top + \lambda I$.
Then, we have
\begin{align*}
    \EE_{d^{(k)}}f^{(k)}(x_{h+1})
    &= \EE_{d^{(k)}}\int_{x_h}\phi^\star_{h}(x_{h},a_{h})^\top\diff\mu_{h}^\star(x_{h+1}) f^{(k)}(x_{h+1})
    \\&=\EE_{d^{(k)}}\phi^\star_{h}(x_{h},a_{h}) ^\top \int_{x_{h+1}}f^{(k)}(x_{h+1})\diff\mu_{h}^\star(x_{h+1}).
    \\&\leq \|\EE_{d^{(k)}}\phi^\star_{h}(x_{h},a_{h})\|_{\Sigma_k^{-1}} \|\int_{x_{h+1}}f^{(k)}(x_{h+1})\diff\mu_{h}^\star(x_{h+1})\|_{\Sigma_k}.
\end{align*}
Focusing on the second term,
\begin{align*}
    \|\int_{x_{h+1}}f^{(k)}(x_{h+1})\diff\mu_{h}^\star(x_{h+1})\|_{\Sigma_k}^2
    &= \sum_{i=1}^{k-1}\prns{ \Eb[d^{(i)}]{ f^{(k)}(x_{h+1}) } }^2 + \lambda d
\end{align*}
Thus, we have shown that
\begin{align*}
    \EE_{d^{(k)}}f^{(k)}(x_{h+1})\leq \|\EE_{d^{(k)}}\phi^\star_{h}(x_{h},a_{h})\|_{\Sigma_k^{-1}} \sqrt{ \sum_{i=1}^{k-1} \prns{\Eb[d^{(i)}]{ f^{(k)}(x_{h+1}) } }^2 + \lambda d }.
\end{align*}

Then, by the independent sequence assumption, we have
\begin{align*}
    T\eps
    &< \sum_{k=1}^T \EE_{d^{(k)}}f^{(k)}(x_{h+1}) \leq \sum_{k=1}^T\|\EE_{d^{(k)}}\phi^\star_{h}(x_{h},a_{h})\|_{\Sigma_k^{-1}} \sqrt{ \sum_{i=1}^{k-1} \prns{\Eb[d^{(i)}]{ f^{(k)}(x_{h+1}) }}^2 + \lambda d }
    \\&\leq \sum_{k=1}^T\|\EE_{d^{(k)}}\phi^\star_{h}(x_{h},a_{h})\|_{\Sigma_k^{-1}} \prns{\eps+\sqrt{\lambda d}} \tag{$\sqrt{ \sum_{i=1}^{k-1} \prns{\Eb[d^{(i)}]{ f^{(k)}(x_{h+1}) }}^2}\leq \eps$}
    \\&\leq 2\eps\sum_{k=1}^T\|\EE_{d^{(k)}}\phi^\star_{h}(x_{h},a_{h})\|_{\Sigma_k^{-1}} \tag{$\lambda = \eps^2/d$}
    \\&\leq 2\eps\sqrt{T}\sqrt{\sum_{k=1}^T\|\EE_{d^{(k)}}\phi^\star_{h}(x_{h},a_{h})\|_{\Sigma_k^{-1}}^2 }
    \\&\leq 2\eps\sqrt{T}\sqrt{d\log(1+\nicefrac{T}{d\lambda})} \tag{elliptical potential}
    \\&\leq 2\eps\sqrt{T}\sqrt{d\log(1+\nicefrac{T}{\eps^2})}. \tag{$\lambda=\eps^2/d$}
\end{align*}
For a reference of the elliptical potential, see \citet[Lemmas 19\&20]{uehara2021representation}.
Rearranging, we have $\sqrt{T} < 2\sqrt{d\log(1+\nicefrac{T}{\eps^2})}$, which implies
$$T \leq 4d\log(1+\nicefrac{T}{\eps^2}).$$
By applying \cref{lem:invert-log-inequality}, we have $T\leq 24d\log\prns{1+\nicefrac{4d}{\eps^2}}$.
This concludes the proof.
\end{proof}

\begin{lemma}\label{lem:invert-log-inequality}
Let $c_1,c_2 \geq 1$ be constants.
Let $x\geq 0$ be a solution to $x\leq c_1\log(1+c_2x)$. Then, we necessarily have $x \leq 6c_1\log\prns{1+c_1c_2}$.
\end{lemma}
\begin{proof}
Using change of variables $B = \frac{x}{c_1}$, we have the inequality is equivalent to $B\leq \log(1+B\cdot c_1c_2)$.
Take $\exp$ of both sides to get $\exp(B)\leq \alpha B + 1$ where $\alpha = c_1c_2$.
From Step 3 of the proof of \citet[Proposition 6]{russo2013eluder}, we have $B\leq \frac{e}{e-1}\frac{e}{e-1}\prns{ \log(1+\alpha) + \log(e/(e-1)) } \leq 3\prns{\log(1+c_1c_2) + 1}$.
Hence, $x\leq c_1 \cdot 3\prns{\log(1+c_1c_2) + 1}.$
\end{proof}

\subsection{Bounding Q-type $\ell_2$ eluder dimension in tabular MDPs}

\begin{theorem}[Bound of $\ell_2$ distributional eluder for tabular MDPs]\label{thm:tabular-mdp-eluder}
Suppose the MDP is a tabular MDP.
Let $\Psi\subset \Xcal\times\Acal\to[0,1]$ be any class of functions mapping $\Xcal\times\Acal$ to $[0,1]$.
Suppose $\Dcal$ be any set of distributions.
Then, we have
\begin{equation}
    \DE_2(\Psi,\Dcal,\eps)\leq\Ocal(SA\log(SA/\eps)).
\end{equation}
\end{theorem}
\begin{proof}
Suppose $\braces{d^{(k)},f^{(k)}}_{k\in[T]}$ is any sequence such for all $k\in[T]$, we have that $(d^{(k)},f^{(k)})$ is $(\eps,\ell_2)$-independent of its predecessors.
Since the MDP is tabular, we can interpret $d^{(k)},f^{(k)}$ as $SA$-dimensional vectors.
For any $k$, set $\Sigma_k = \sum_{i=1}^{k-1}d^{(i)}(d^{(i)})^\top + \lambda I$.
Then, we have
\begin{align*}
    \EE_{d^{(k)}}f^{(k)}(x,a) = (d^{(k)})^\top f^{(k)} \leq \|d^{(k)}\|_{\Sigma_k^{-1}}\|f^{(k)}\|_{\Sigma_k}.
\end{align*}
Focusing on the second term, we have
\begin{align*}
    \|f^{(k)}\|_{\Sigma_k}^2
    &= \sum_{i=1}^{k-1}\prns{ (d^{(i)})^\top f^{(k)} }^2 + \lambda SA.
\end{align*}
Thus, we have
\begin{align*}
    T\eps
    &< \sum_{k=1}^T \EE_{d^{(k)}}f^{(k)}(x,a)
    \leq \sum_{k=1}^T \|d^{(k)}\|_{\Sigma_k^{-1}} \sqrt{ \sum_{i=1}^{k-1}\prns{ \EE_{d^{(i)}}[f^{(k)}(x,a)] }^2 + \lambda SA }
    \\&\leq \sum_{k=1}^T \|d^{(k)}\|_{\Sigma_k^{-1}} \prns{ \eps + \sqrt{\lambda SA} }
    \\&\leq 2\eps \sum_{k=1}^T \|d^{(k)}\|_{\Sigma_k^{-1}} \tag{$\lambda=\eps^2/SA$}
    \\&\leq 2\eps\sqrt{T}\sqrt{\sum_{k=1}^T \|d^{(k)}\|_{\Sigma_k^{-1}}^2 }
    \\&\leq 2\eps\sqrt{T}\sqrt{SA \log(1 + T/\eps^2)}. \tag{elliptical potential}
\end{align*}
Rearranging, we have $\sqrt{T}<2\sqrt{SA\log(1+T/\eps^2)}$, which implies $T\leq 4SA\log(1+T/\eps^2)$. Then by applying \cref{lem:invert-log-inequality}, we have $T\leq 24SA\log\prns{1+\nicefrac{4SA}{\eps^2}}$.
This concludes the proof.
\end{proof}

\newpage
\section{Proofs for Online RL}\label{app:proofs-online-rl}
\subsection{Preliminary Lemmas}

\begin{lemma}\label{lem:dist-bellman-expectation}
For any policy $\pi$, conditional distribution $d$ and $h\in[H]$, we have
\begin{align*}
    &\overline{\Tcal_h^{\pi,D}d(x,a)} = \Tcal_h^\pi \bar d(x,a),
    \\&\overline{\Tcal_h^{\star,D}d(x,a)} = \Tcal_h^\star \bar d(x,a).
\end{align*}
\end{lemma}
\begin{proof}
\begin{align*}
    \overline{\Tcal_h^{\pi,D}d(x,a)}
    &= \Eb[y\sim \Tcal_h^{\pi,D}d(x,a)]{y}
    \\&= \Eb[c\sim C_h(x,a), x'\sim P_h(x,a), a'\sim\pi_{h+1}(x'),y'\sim d(x',a')]{c+y'}
    \\&= \bar C_h(x,a) + \Eb[x'\sim P_h(x,a), a'\sim\pi_{h+1}(x'),y'\sim d(x',a')]{ y' }
    \\&= \bar C_h(x,a) + \Eb[x'\sim P_h(x,a), a'\sim\pi_{h+1}(x')]{\bar d(x',a')}
    \\&= \Tcal_h^\pi \bar d(x,a).
\end{align*}
\begin{align*}
    \overline{\Tcal_h^{\star,D}d(x,a)}
    &= \Eb[y\sim \Tcal_h^{\star,D}d(x,a)]{y}
    \\&= \Eb[c\sim C_h(x,a), x'\sim P_h(x,a), a'=\argmin_{\wt a} \bar d(x',\wt a),y'\sim d(x',a')]{c+y'}
    \\&= \bar C_h(x,a) + \Eb[x'\sim P_h(x,a), a'=\argmin_{\wt a} \bar d(x',\wt a),y'\sim d(x',a')]{ y' }
    \\&= \bar C_h(x,a) + \Eb[x'\sim P_h(x,a), a'=\argmin_{\wt a} \bar d(x',\wt a)]{\bar d(x',a')}
    \\&= \bar C_h(x,a) + \Eb[x'\sim P_h(x,a)]{\min_{a'}\bar d(x',a')}
    \\&= \Tcal_h^\star \bar d(x,a).
\end{align*}
\end{proof}

\begin{lemma}[Performance Difference Lemma (PDL)]\label{lem:pdl}
For any $f:\prns{\Xcal\times\Acal\to\RR}^H$ and policies $\pi,\pi'$, we have
\begin{align}
    V^\pi-\Eb[a\sim\pi'(x_1)]{f_1(x_1,a)} = \sum_{h=1}^H \Eb[\pi]{ \Tcal^{\pi'}_h f_{h+1}(x_h,a_h) - f_h(x_h,\pi') }.
\end{align}
\end{lemma}
\begin{proof}
We proceed by inducting on the following claim: for all $h=H+1,H,\dots,1$,
\begin{align*}
    V^\pi_h(x_h)-f_h(x_h,\pi') = \sum_{t=h}^H\Eb[\pi,x_h]{ \Tcal^{\pi'}_t f_{t+1}(x_t,a_t) - f_t(x_t,\pi') }.
\end{align*}
The base case of $H+1$ is trivially true as everything is $0$.
Now fix any $h$ and suppose the IH at $h+1$ is true. Then
\begin{align*}
    &V^\pi_h(x_h)-f_h(x_h,\pi')
    \\&=\Eb[\pi,x_h]{ c_h + V^\pi_{h+1}(x_{h+1}) - f_{h+1}(x_{h+1},\pi') + f_{h+1}(x_{h+1},\pi') - f_h(x_h,\pi') }
    \\&=\Eb[\pi,x_h]{ V^\pi_{h+1}(x_{h+1}) - f_{h+1}(x_{h+1},\pi') } + \Eb[\pi,x_h]{ c_h + f_{h+1}(x_{h+1},\pi') - f_h(x_h,\pi') }.
\end{align*}
By the IH, the first term is equal to $\sum_{t=h+1}^H\Eb[\pi,x_h]{ \Tcal^{\pi'}_t f_{t+1}(x_t,a_t) - f_t(x_t,\pi') }$.
The second term is exactly $\Eb[\pi,x_h]{\Tcal^{\pi'}_hf_{h+1}(x_h,a_h) - f_h(x_h,\pi')}$, which concludes the proof.
\end{proof}

\subsection{Proof of Small-Loss Regret and PAC Bounds}

Recall that we defined the function class and distribution class, for each $h$, as
\begin{align}
    &\Dcal_h(\Pi) = \braces{(x,a)\mapsto d^\pi_h(x,a): \pi\in\Pi} \label{eq:lsec-dcal}
    \\&\Psi_h = \braces{(x,a)\mapsto D_\triangle(f(x,a)\Mid \Tcal^{\star,D}f(x,a)): f\in\Fcal}. \nonumber
\end{align}

Also, define the `$V$-type' analogs as follows, which will be useful for PAC instead of regret bounds.
\begin{align}
    &\Dcal_{h,v}(\Pi) = \braces{ x\mapsto d^\pi_h(x):\pi\in\Pi } \label{eq:lsec-dcal-v}
    \\&\Psi_{h,v} = \braces{ x\mapsto \EE_{a\sim \op{Unif}(\Acal)}[D_\triangle(f(x,a)\Mid \Tcal^{\star,D}f(x,a))] : f\in\Fcal }. \nonumber
\end{align}
Let us also overload notation for the eluder dimensions as
\begin{align*}
    &\DE_1(\eps) := \max_h \DE_1(\Psi_h,\Dcal_h(\Pi),\eps),
    \\&\DE_{1,v}(\eps) := \max_h \DE_1(\Psi_{h,v},\Dcal_{h,v}(\Pi),\eps).
\end{align*}

Before we prove the following main theorem, a couple of remarks are in order:
\begin{enumerate}[leftmargin=0.7cm]
    \item Recall that by \cref{thm:tabular-mdp-eluder}, we have $\DE_1(\eps)\leq\Ocal(SA\log(SA/\eps))$ and by \cref{thm:low-rank-mdp-eluder}, we have $\DE_{1,v}(\eps)\leq\Ocal(d\log(d/\eps))$. This shows that the Eluder dimension in terms in \cref{thm:online-general} are appropriately bounded.
    \item In \cref{sec:dist-bc-low-rank-mdps}, we showed that distributional BC (\cref{ass:policy-dist-bellman-completeness}) is satisfied in low-rank MDPs and the log bracketing number is bounded by $\Ocal(dM\log(d/\eps) + \log|\Phi|)$ where $\Phi$ is a realizable class for $\phi^\star$. This shows that the BC assumption of \cref{thm:online-general} is satisfied and $\beta$ is appropriately bounded for low-rank MDPs.
\end{enumerate}
Taken together, these two points imply that we have a small-loss PAC bound for low-rank MDPs: concretely, we have $V^{\bar\pi}-V^\star\leq \wt\Ocal\prns{dH\sqrt{\frac{AV^\star \log|\Phi|}{K}} + \frac{d^2H^2A\log|\Phi|}{K}}$.

We now prove the our main result for online RL: \cref{thm:online-general}. We will prove the result with general function classes, so we will replace the $|\Fcal|$ by its $\ell_\infty$ bracketing number, \ie, $\beta=\log(HKN_{[]}(1/K,\Fcal,\ell_\infty)/\delta)$.
\onlineRLGeneral*
\begin{proof}
For shorthand, let $\delta_{h,k}(x,a) := D_\triangle(f_h^{(k)}(x,a)\Mid \Tcal^{\star,D}_h f_{h+1}^{(k)}(x,a))$ and $\Delta_k:=\sum_{h=1}^H\Eb[\pi^k]{ \delta_{h,k}(x_h,a_h) }$. Notice that since $\pi^k_{h+1}(x) = \argmin_a \bar f^{(k)}_{h+1}(x,a)$, we have $\Tcal^{\pi^k,D}_hf^{(k)}_{h+1}(x,a)=\Tcal^{\star,D}_hf^{(k)}_{h+1}(x,a)$, so $\delta_{h,k}(x,a) = D_\triangle(f_h^{(k)}(x,a)\Mid \Tcal^{\pi^k,D}_h f_{h+1}^{(k)}(x,a))$ as well.

By \cref{thm:distgolf-mle}, we have the following two facts
\textbf{for all $k\in[K]$}, \\
(i) Optimism: $\min_a\bar f^{(k)}_1(x_1,a)\leq V^\star$ (since $Z^\star\in\Fcal_k$) and \\
(ii) Low training error: for all $h$, we have
\begin{enumerate}[leftmargin=3cm]
    \item[If \textsc{UAE}=\textsc{False}.] $\sum_{i<k}\Eb[\pi^i]{\delta_{h,k}(s_h,a_h)}\leq 240\beta$.
    \item[If \textsc{UAE}=\textsc{True}.] $\sum_{i<k}\Eb[\pi^i]{\EE_{a'\sim\op{unif}(\Acal)}[\delta_{h,k}(s_h,a_h)]}\leq 240\beta$.
\end{enumerate}
The $240$ comes from the constants of \cref{thm:distgolf-mle} and the fact that $D_\triangle(a,b)\leq 4H^2(a,b)$ for all distributions $a,b$.

Now, fix any episode $k\in[K]$.
\begin{align*}
    &V^{\pi^k}-V^\star
    \\&\leq V^{\pi^k}-\min_a\bar f^{(k)}_1(x_1,a) \tag{Fact (i)}
    \\&=\sum_{h=1}^H \Eb[\pi^k]{ \Tcal^{\pi^k}_h\bar f^{(k)}_{h+1}(x_h,a_h) - \bar f^{(k)}_h(x_h,\pi^k_h(x_h)) } \tag{PDL \cref{lem:pdl}}
    \\&=\sum_{h=1}^H \Eb[\pi^k]{ \overline{\Tcal^{\pi^k,D}_h f^{(k)}_{h+1}}(x_h,a_h) - \bar f^{(k)}_h(x_h,a_h) } \tag{\cref{lem:dist-bellman-expectation}}
    \\&\leq\sum_{h=1}^H\sqrt{ \Eb[\pi^k]{ 4\bar f^{(k)}_h(x_h,a_h) + \delta_{h,k}(x_h,a_h) } } \cdot \sqrt{\Eb[\pi^k]{ \delta_{h,k}(x_h,a_h) }} \tag{\cref{eq:tri-disc-ineq-2}}
    \\&\leq\sum_{h=1}^H\sqrt{  4eV^{\pi^k} + 17H\sum_{t=h}^H\Eb[\pi^k]{ \delta_{t,k}(x_t,a_t)} } \cdot \sqrt{\Eb[\pi^k]{ \delta_{h,k}(x_h,a_h) }} \tag{\cref{lem:self-bounding} and $\Eb[\pi]{Q_h^{\pi}(s_h,a_h)}\leq V^\pi$}
    \\&\leq \sqrt{ 4eV^{\pi^k} + 17H\Delta_k } \cdot \sqrt{H\Delta_k } \tag{$\bigstar$}
    \\&\leq \sqrt{4eHV^{\pi^k} \Delta_k} + 5H\Delta_k
    \\&\leq 2\sqrt{H}\eta^{-1}V^{\pi^k} + 2\sqrt{H}\eta\Delta_k + 5H\Delta_k.
\end{align*}
In $\bigstar$, we used Cauchy Schwartz. Setting $\eta = 4\sqrt{H}$ and rearranging, we have
\begin{align*}
    V^{\pi^k}\leq 2V^\star + 16H\Delta_k + 10H\Delta_k \leq 2V^\star + 26H\Delta_k.
\end{align*}
Plugging this into $\bigstar$, and noting $104e+17\leq 300$, we have
\begin{align*}
    V^{\pi^k}-V^\star
    &\leq \sqrt{ 8eV^\star + 300H\Delta_k }\sqrt{H\Delta_k}.
\end{align*}
Thus, summing the instantaneous regrets over all episodes, we get
\begin{align*}
    \sum_{k=1}^KV^{\pi^k}-V^\star
    &\leq \sum_{k=1}^K\sqrt{ 8eV^\star + 300H\Delta_k }\sqrt{H\Delta_k}
    \\&\leq \sqrt{ 8eKV^\star + 300H\sum_k\Delta_k }\sqrt{H\sum_k\Delta_k} \tag{Cauchy-Schwartz}
    \\&\leq 5\sqrt{HKV^\star\sum_k\Delta_k} + 18H\sum_k\Delta_k.
\end{align*}

\paragraph{Last step: bounding $\sum_k \Delta_k$.}
In this final step, we invoke the pigeonhole property of the eluder dimension, as proven in \cref{thm:eluder-pigeonhole}. Note that the precondition of \cref{thm:eluder-pigeonhole} is satisfied by Fact (ii) mentioned at the beginning of this proof. Also, since the triangular discrimination is always bounded by $1$, we have that $C$ in \cref{thm:eluder-pigeonhole} is at most $1$, and we will also pick $\eps=1/K$.

On one hand, if \textsc{UAE}=\textsc{False}, then,
\begin{align*}
    \sum_{k=1}^K\Delta_k
    = \sum_{h=1}^H\sum_{k=1}^K\Eb[\pi^k]{\delta_{h,k}(x_h,a_h)}
    \leq 1000H \DE_1(1/K)\beta\log(K).
\end{align*}

On the other hand, if \textsc{UAE}=\textsc{True}, then, we use the V-type analogs,
\begin{align*}
    \sum_{k=1}^K\Delta_k
    &= \sum_{h=1}^H\sum_{k=1}^K\Eb[\pi^k]{\delta_{h,k}(x_h,a_h)}
    \\&\leq A\sum_{h=1}^H\sum_{k=1}^K\Eb[\pi^k]{\EE_{a\sim\op{unif}(\Acal)}\delta_{h,k}(x_h,a)}
    \\&\leq 1000AH \DE_1(1/K)\beta\log(K).
\end{align*}
This concludes the proof for both the regret and PAC bounds.
\end{proof}

\begin{lemma}[Self-bounding lemma]\label{lem:self-bounding}
Let $f\in\Fcal$ and let $\pi$ be any policy. Let us denote $\delta_h(x,a):=D_\triangle(f_h(x,a)\Mid \Tcal^{\pi,D}_h f_{h+1}(x,a))$.
Then, for all $h\in[H]$, for all $x_h,a_h$, we have
\begin{align*}
    \bar f_h(x_h,a_h)\leq eQ^\pi_h(x_h,a_h) + 4H\sum_{t=h}^H\Eb[\pi,x_h,a_h]{\delta_t(x_t,a_t)}.
\end{align*}
\end{lemma}
\begin{proof}
We prove the following refined subclaim inductively: for all $h\in[H]$, for all $x_h,a_h$, we have
\begin{align}
    \bar f_h(x_h,a_h)\leq \sum_{t=h}^H \prns{1+\frac{1}{H}}^{t-h}\Eb[\pi,x_h,a_h]{\bar c_t(x_t,a_t) + 2H\delta_t(x_t,a_t) }. \tag{IH}
\end{align}
For $H+1$ this is trivially true.
Now fix any $h$ and suppose IH is true for $h+1$.
By \cref{eq:tri-disc-ineq-2}, for any $h,x_h,a_h$, we have,
\begin{align*}
    \bar f_h(x_h,a_h)-\Tcal^{\pi}_h\bar f_{h+1}(x_h,a_h)
    &\leq \sqrt{4\Tcal^{\pi}_h\bar f_{h+1}(x_h,a_h) + \delta_h(x_h,a_h)}\sqrt{ \delta_h(x_h,a_h) }
    \\&\leq \sqrt{4\Tcal^{\pi}_h\bar f_{h+1}(x_h,a_h)\delta_h(x_h,a_h)} + \delta_h(x_h,a_h)
    \\&\leq \frac{1}{H}\Tcal^{\pi}_h\bar f_{h+1}(x_h,a_h) + (H+1)\delta_h(x_h,a_h). \tag{AM-GM}
\end{align*}
In particular, we have that
\begin{align*}
    &\bar f_h(x_h,a_h)
    \\&\leq\prns{1+\frac{1}{H}} \Tcal^{\pi}_h\bar f_{h+1}(x_h,a_h) + 2H\delta_h(x_h,a_h)
    \\&= \prns{1+\frac{1}{H}} \prns{\bar c_h(x_h,a_h) + \Eb[x_{h+1}\sim P_h^\star(x_h,a_h)]{\bar f_{h+1}(x_{h+1},\pi) } } + 2H\delta_h(x_h,a_h)
    \\&\leq\prns{1+\frac{1}{H}} \prns{\bar c_h(x_h,a_h) + \Eb[x_{h+1}\sim P_h^\star(x_h,a_h)]{\sum_{t=h+1}^H\prns{1+\frac{1}{H}}^{t-h-1}\Eb[\pi,x_{h+1}]{\bar c_t(x_t,a_t) + 2H\delta_t(x_t,a_t)}} }  \tag{IH}
    \\&+ 2H\delta_h(x_h,a_h),
\end{align*}
which proves the inductive claim.
Noting that $\sum_{t=1}^H\prns{1+1/H}^t\leq e$, we have proven the lemma.
\end{proof}

\subsection{Regret Bounds for Tabular MDPs}\label{sec:tabular-regret-bounds}
\begin{restatable}[Small-loss regret for tabular MDP]{theorem}{tabularRegret}\label{thm:online-tabular-regret}
Suppose the MDP is tabular with $X$ states and assume \cref{ass:policy-dist-bellman-completeness}.
Fix any $\delta\in(0,1)$ and set $\beta=\log(HK\abs{\Fcal}/\delta)$.
Then, w.p. at least $1-\delta$, %
\begin{equation*}
    \op{Regret}_{\online{}}(K) \in\Ocal\prns*{ H\sqrt{XAKV^\star\beta } + H^2XA\beta }.
\end{equation*}
\end{restatable}
In terms of $H,X,A,K$ scaling, our bound matches that of GOLF \citep{xie2023the} and
is only a $H$ factor looser than that of the minimax lower bound $\wt\Ocal(\sqrt{XAK})$.
The key benefit over prior bounds is that our leading term scales with the minimum cost of the problem $V^\star$.
For example, if $V^\star\approx 0$, \online{} attains $\Ocal(\log K)$ regret while uniform regret bounds are lower bounded by $\Omega(\sqrt{K})$.
Compared to the minimax-optimal UCBVI \citep{azar2017minimax}, one weakness of our theorem is that it needs a $\Fcal$ satisfying BC.
Fortunately, in tabular MDPs where cost is only revealed at the last step from a known distribution,
we can choose $\Fcal_{tab}$ as described in \citet[Lemma 4.15]{wu2023distributional} to automatically satisfy BC.
By extending our theory via bracketing entropy (\cref{sec:confidence-set-infinite-functions}), we can derive that $\Fcal_{tab}$ yields $\beta=\Ocal(X^2A^2\log(XAHK/\delta))$.
We note that if costs are unknown but discrete, it is possible to construct a BC function class with $\beta$ scaling as $\Ocal(X^2A^2\log(nXAHK/\delta))$ where $n$ is the maximum number of possible cumulative costs.

\paragraph{Extension to linear MDPs}
The Q-type dimension captures Linear MDPs when squared loss is used by exploiting the fact that the bellman residual is linear in $\phi^\star(x,a)$ \citep{jin2021bellman}. However, since our function class is the set of triangular discriminations, rather than the Bellman residual, we find that the Q-type dimension does not immediately capture Linear MDPs unless regularity assumptions are made. For instance, we believe that Linear MDPs are captured by the Q-type dimension if we assume that $Z^\pi_h(z\mid x,a)$ is lower bounded, \ie, the value distribution is sufficiently smooth.

\section{Proofs for Offline RL}
\offlinePAC*
\begin{proof}[Proof of \cref{thm:offline}]
For shorthand, let $\delta_h^\pi(x,a) = D_\triangle(f_h^{\pi}(x,a)\Mid \Tcal^{\pi,D}_hf^{\pi}_{h+1}(x,a))$ and $\Delta^\pi = \sum_{h=1}^H\Eb[\pi]{\delta_h^\pi(x_h,a_h)}$. Also, let $f(x,\pi) = \Eb[a\sim\pi(x)]{f(x,a)}$.

By \cref{thm:distbco-mle}, we have the following two facts, for all $\pi\in\Pi$, \\
(i) Pessimism: $V^\pi\leq \bar f^\pi_1(x_1,\pi)$ (since $Z^\pi\in\Fcal_\pi$) for all $\pi\in\Pi$, and \\
(ii) $\Eb[\nu_h]{\delta_{h}^\pi(x_h,a_h)}\leq \beta' N^{-1}$ for all $h$ where \cref{thm:distbco-mle} and the fact that $D_\triangle\leq 4H^2$ certifies that $\beta'=240\beta$ is sufficient.

With these two facts, we can bound the suboptimality of $\wh\pi$ as follows:
\begin{align*}
    &V^{\wh\pi}-V^{\wt\pi}
    \\&\leq \bar f^{\wh\pi}_1(x_1,\wh\pi)-V^{\wt\pi} \tag{Fact (i)}
    \\&\leq \bar f^{\wt\pi}_1(x_1,\wt\pi)-V^{\wt\pi} \tag{Policy selection scheme in \cref{alg:offline} (\cref{line:distbco-policy-selection})}
    \\&= \sum_{h=1}^H \Eb[\wt\pi]{ \bar f^{\wt\pi}_h(x_h,\wt\pi)-\Tcal_h^{\wt\pi} \bar f^{\wt\pi}_{h+1}(x_h,a_h)  } \tag{PDL \cref{lem:pdl}}
    \\&\leq \sum_{h=1}^H\sqrt{ \Eb[\wt\pi]{ 4\bar f^{\wt\pi}_h(x_h,a_h) + \delta^{\wt\pi}_h(x_h,a_h)} } \sqrt{\Eb[\wt\pi]{\delta^{\wt\pi}_h(x_h,a_h)}} \tag{\cref{eq:tri-disc-ineq-2}}
    \\&\leq \sum_{h=1}^H\sqrt{ 4eV^{\wt\pi} + 17H\sum_{t=h}^H\Eb[\wt\pi]{ \delta^{\wt\pi}_t(x_t,a_t)} } \sqrt{\Eb[\wt\pi]{\delta^{\wt\pi}_h(x_h,a_h)}} \tag{\cref{lem:self-bounding}}
    \\&\leq \sqrt{ 4eV^{\wt\pi} + 17H\Delta^{\wt\pi}} \sqrt{H\Delta^{\wt\pi}}
    \\&\leq 4\sqrt{HV^{\wt\pi}\Delta^{\wt\pi}} + 5H\Delta^{\wt\pi}.
\end{align*}
Finally, we can bound $\Delta^{\wt\pi}$ by a change of measure,
\begin{align*}
    \Delta^{\wt\pi}
    &= \sum_{h=1}^H\Eb[\wt\pi]{\delta_h^{\wt\pi}(x_h,a_h)}
    \\&\leq C^{\wt\pi} \sum_{h=1}^H\Eb[\nu_h]{\delta_h(x_h,a_h)}
    \\&\leq C^{\wt\pi} H \cdot \beta' N^{-1}. \tag{Fact (ii)}
\end{align*}
Therefore,
\begin{align*}
    V^{\wh\pi}-V^{\wt\pi}\leq 4H\sqrt{\frac{C^{\wt\pi} V^{\wt\pi}\beta'}{N} } + \frac{5H^2C^{\wt\pi}\beta'}{N}.
\end{align*}
\end{proof}

\newpage
\section{Extension: Small-Return Bounds}\label{sec:small-reward}
In this section, we show that \online{} and \offline{} can also be used to obtain small-return bounds.
Compared to the algorithms presented in the main text for minimizing cost, we simply have to replace $\min$ with $\max$ (and vice versa) for maximizing reward, \ie, see \cref{app:omitted-algs} and enable the \textsc{SmallReturn} flag.
The proofs are also largely the same, with slight changes to the first few steps.

\begin{theorem}\label{thm:online-general-reward-maximization}
Assume \cref{ass:policy-dist-bellman-completeness} and suppose we want to maximize returns (instead of minimize cost), so enable the \textsc{SmallReturn} flag.
Fix any $\delta\in(0,1)$ and set $\beta = \log(HK|\Fcal|/\delta)$ and $\beta'=60\beta$.
Then, w.p. at least $1-\delta$, running \online{} (\cref{alg:onlinerl_appendix}) with $\textsc{UAE}=\textsc{False}$ yields the following small-loss regret bound,
\begin{align}
\op{Regret}_{\online{}}(K)\leq 5H\sqrt{KV^\star \LSEC(K) \beta'} + 18H^2\LSEC(K) \beta'.
\end{align}
If instead $\textsc{UAE}=\textsc{True}$, the outputted policy $\bar\pi$ enjoys the following small-loss PAC bound,
\begin{align*}
    V^\star-V^{\bar\pi}\leq 5H\sqrt{\frac{AV^\star \LSEC_v(K) \beta'}{K}} + 18H^2\frac{A\LSEC_v(K)\beta'}{K}.
\end{align*}
\end{theorem}
\begin{proof}
Adopt the same notation as in the proof of \cref{thm:online-general}.
By \cref{thm:distgolf-mle}, we have the following two facts
for all $k\in[K]$, \\
(i) Optimism: $V^\star\leq \max_a\bar f^{(k)}_1(x_1,a)$ (since $Z^\star\in\Fcal_k$) and \\
(ii) $\sum_{i<k}\Eb[\pi^i]{\delta_{h,k}(s_h,a_h)}\leq\beta'$ for all $h$. If \textsc{UAE}=\textsc{True}, then $a_h$ is sampled from $\op{unif}(\Acal)$ rather than $\pi^i$, \ie, we have $\sum_{i<k}\Eb[s_h\sim\pi^i,a_h\sim \op{unif}(\Acal)]{\delta_{h,k}(s_h,a_h)}\leq\beta'$, where $\beta'\lesssim \beta$. \cref{thm:distgolf-mle} certifies that $\beta'=60\beta$ is sufficient.

Fix any episode $k\in[K]$. Then,
\begin{align*}
    &V^\star-V^{\pi^k}
    \\&\leq \max_a\bar f^{(k)}_1(x_1,a)-V^{\pi^k} \tag{Fact (i)}
    \\&=\sum_{h=1}^H \Eb[\pi^k]{ \bar f^{(k)}_h(x_h,\pi^k_h(x_h))-\Tcal^{\pi^k}_h\bar f^{(k)}_{h+1}(x_h,a_h) } \tag{PDL \cref{lem:pdl}}
    \\&=\sum_{h=1}^H \Eb[\pi^k]{ \bar f^{(k)}_h(x_h,a_h)-\overline{\Tcal^{\pi^k,D}_h f^{(k)}_{h+1}}(x_h,a_h) } \tag{\cref{lem:dist-bellman-expectation}}
    \\&\leq\sum_{h=1}^H\sqrt{ \Eb[\pi^k]{ 4\bar f^{(k)}_h(x_h,a_h) + \delta_{h,k}(x_h,a_h) } } \cdot \sqrt{\Eb[\pi^k]{ \delta_{h,k}(x_h,a_h) }} \tag{\cref{eq:tri-disc-ineq-2}}
    \\&\leq\sum_{h=1}^H\sqrt{  4eV^{\pi^k} + 17H\sum_{t=h}^H\Eb[\pi^k]{ \delta_{t,k}(x_t,a_t)} } \cdot \sqrt{\Eb[\pi^k]{ \delta_{h,k}(x_h,a_h) }} \tag{\cref{lem:self-bounding} and $\Eb[\pi]{Q_h^{\pi}(s_h,a_h)}\leq V^\pi$}
    \\&\leq \sqrt{ 4eV^{\pi^k} + 17H\Delta_k } \cdot \sqrt{H\Delta_k } \tag{$\clubsuit$}
    \\&\leq \sqrt{ 4eV^{\star} + 17H\Delta_k } \cdot \sqrt{H\Delta_k }
\end{align*}
Thus, summing the instantaneous regrets over all episodes, we get
\begin{align*}
    \sum_{k=1}^KV^{\pi^k}-V^\star
    &\leq \sum_{k=1}^K\sqrt{ 4eV^\star + 17H\Delta_k }\sqrt{H\Delta_k}
    \\&\leq \sqrt{ 4eKV^\star + 17H\sum_k\Delta_k }\sqrt{H\sum_k\Delta_k} \tag{Cauchy-Schwartz}
    \\&\leq 5\sqrt{HKV^\star\sum_k\Delta_k} + 18H\sum_k\Delta_k.
\end{align*}
The bounds for $\Delta_k$ are the same as in \cref{thm:online-general}.
\end{proof}
In some sense, the proof for the small-returns bound is actually easier than the small-loss bound.
Recall that in the cost-minimizing setting, we needed to perform a crucial Cauchy-Schwartz step to rearrange terms at the step labelled $\clubsuit$.
However, in the reward-maximizing setting, we simply bound $V^{\pi^k}\leq V^\star$, without needing to rearrange terms.

\begin{theorem}\label{thm:offline-small-reward}
Assume \cref{ass:policy-dist-bellman-completeness} and suppose we want to maximize returns (instead of minimize cost), so enable the \textsc{SmallReturn} flag.
Fix any $\delta\in(0,1)$ and set $\beta=\log(H|\Pi||\Fcal|/\delta)$.
Then, w.p. at least $1-\delta$, \offline{} (\cref{alg:onlinerl_appendix}) learns a policy $\wh\pi$
such that for any comparator policy $\wt\pi\in\Pi$, we have
\begin{align*}
    V^{\wt\pi}-V^{\wh\pi}\leq 9H\sqrt{\frac{C^{\wt\pi} V^{\wt\pi}\beta}{N} } + \frac{30H^2C^{\wt\pi}\beta}{N}.
\end{align*}
\end{theorem}
\begin{proof}[Proof of \cref{thm:offline-small-reward}]
Adopt the same notation as in the proof of \cref{thm:offline}.
By \cref{thm:distbco-mle}, we have the following two facts, for all $\pi\in\Pi$, \\
(i) Pessimism: $\bar f^\pi_1(x_1,\pi)\leq V^\pi$ (since $Z^\pi\in\Fcal_\pi$) for all $\pi\in\Pi$, and \\
(ii) $\Eb[\nu_h]{\delta_{h}^\pi(x_h,a_h)}\leq \beta' N^{-1}$ for all $h$ where $\beta'\leq60\beta$.

With these two facts, we can bound the suboptimality of $\wh\pi$ as follows:
\begin{align*}
    &V^{\wt\pi}-V^{\wh\pi}
    \\&\leq V^{\wt\pi}-\bar f^{\wh\pi}_1(x_1,\wh\pi) \tag{Fact (i)}
    \\&\leq V^{\wt\pi}-\bar f^{\wt\pi}_1(x_1,\wt\pi) \tag{Policy selection rule in \cref{line:distbco-policy-selection-appendix}}
    \\&= \sum_{h=1}^H \Eb[\wt\pi]{ \Tcal_h^{\wt\pi} \bar f^{\wt\pi}_{h+1}(x_h,a_h)-\bar f^{\wt\pi}_h(x_h,\wt\pi)  } \tag{PDL \cref{lem:pdl}}
    \\&\leq \sum_{h=1}^H\sqrt{ \Eb[\wt\pi]{ 4\bar f^{\wt\pi}_h(x_h,a_h) + \delta^{\wt\pi}_h(x_h,a_h)} } \sqrt{\Eb[\wt\pi]{\delta^{\wt\pi}_h(x_h,a_h)}} \tag{\cref{eq:tri-disc-ineq-2}}.
\end{align*}
From here, the same argument in the proof of \cref{thm:offline} finishes the proof.
\end{proof}

\section{Experiment Details}\label{sec:experiment-details}

\textbf{Experiment Settings}

In our experiments, as outlined in \citet{foster2021efficient}, our $\gamma$ learning rate at each time step $t$ is set to $\gamma_{t} = \gamma_{0}t^{p}$ where $\gamma_{0}$ and $p$ are hyperparameters. We use batch sizes of $32$ samples per episode, and the King County and Prudential experiments run for $5,000$ episodes while the CIFAR-100 experiment runs for $15,000$.

For each dataset, we select the hyperparameter configuration with the best performance for each algorithm. As we report two metrics, performance over the last $100$ episodes and over all episodes, we choose the best hyperparameters for each metric as well. While it is often the same hyperparameters that give the best last $100$ episodes and all episodes results for a model, that is not always the case. We use the WandB (Weights and Biases) library to run sweeps over hyperparameters.

\textbf{Oracles}

For our regression oracles, we use ResNet18 \citep{7780459}, with a modified output layer (so that the output is suited for $100$ prediction classes) for CIFAR-100, and a simple $2$ hidden-layer neural network for the Prudential Life Insurance and King County Housing datasets. For \cb{}, the oracle's output layer has size $AC$ where $A$ is the number of actions and $C$ is the number of potential costs. This is reshaped so that for each action, there are predictions associated with each potential cost, which then have a softmax function applied to them to represent cost probabilities. For SquareCB and FastCB, the output size is $A$ because there is just a single prediction associated with each action. As per \citet{foster2021efficient}, a sigmoid function is applied to this output layer. All experiments were implemented using PyTorch.

\textbf{Datasets}

We now provide an overview table as well as additional details and context to our setups for each dataset. Note that the number of items in each dataset in the table is the count after preprocessing.
{
\centering
\begin{table}[!h]
    \centering
    \begin{tabular}{|p{4cm}||p{1.5cm}|p{1.5cm}|p{1.5cm}|}
          \hline
 \multicolumn{4}{|c|}{\textbf{Datasets}} \\
 \hline
 Dataset& Items & Number of Actions & Number of Costs \\
 \hline
 CIFAR-100   & $50,000$    & $100$ & $3$\\
 Prudential Life Insurance&$59,381$ & $8$ & $9$  \\
 King County Housing& $20,148$ &$100$ & $101$ \\
 \hline
    \end{tabular}
    \vspace{0.3cm}
    \caption{Overview of the three datasets and their experimental setups}
    \label{tab:dataset}
\end{table}
}

\textbf{Prudential Life Insurance}
 This dataset is from the Prudential Life Insurance Kaggle competition \citep{prudential-life-insurance-assessment}. It is featured in \citet{farsang2022conditionally}, which inspires our experimental setup. The risk level in $[8]$ directly determines the price charged to the customer. Thus, we can consider the chosen risk level as the action taken. If the model overpredicts the risk level, we get a cost of $1.0$ because this is considered over charging the customer and not getting a sale. Otherwise, the model's prediction is charging too little for the customer. To reiterate, the cost in this case is $.1 *(y - \hat{y})$ where $y$ is the actual risk level, and $\hat{y}$ is the predicted risk level.

\textbf{King County Housing}
The King County housing dataset is also used in \citet{farsang2022conditionally}. An interesting part of the setup is that the cost construction in the case of not overpredicting differs from the Prudential experiment, even though they're both effectively about predicting a price point. Here, the model's chosen price is considered the gain, which is why the cost is $1.0$ minus the chosen price. On the other hand, in the Prudential experiment, the cost is a linear function of the difference between the chosen value and the actual value.

\textbf{CIFAR-100}
For the CIFAR-100 experiment, we use the training dataset of $50,000$ images as our dataset. The inclusion of the superclass is critical, as it lets us delineate $3$ possible costs that \cb{} can learn. Without the super class, the cost construction would be a pure binary of correct vs. incorrect. If this were the case, the ability to test the effectiveness of learning the distribution would be nullified. The distribution would just be whether an action is correct or not, which means our algorithm would essentially be predicting the mean directly.

\textbf{Results}

The largest advantages \cb{} had over the next best algorithm were in the Prudential experiment, with \cb{} having a $.086$ advantage over the last $100$ episodes and a $.045$ advantage over all episodes. While the gaps were not as large for the other two datasets, they are still statistically significant and further showcase the benefit of distribution learning.

\end{document}